\documentclass[11pt]{article}

\usepackage{amsmath}
\usepackage{amsthm}
\usepackage{amssymb}

\usepackage{algorithm}
\usepackage{algpseudocode}

\usepackage{color}
\usepackage[english]{babel}

\usepackage{graphicx}
\usepackage{caption}
\usepackage{subcaption}

\usepackage[square,sort,comma,numbers]{natbib}
\usepackage{wrapfig,epsfig}
\usepackage{psfrag}
\usepackage{epstopdf}
\usepackage{url}
\usepackage{color}
\usepackage{epstopdf}

\usepackage[hidelinks,pdfencoding=auto,psdextra]{hyperref}
\usepackage{hyperref}
\usepackage[margin=1in]{geometry}
\newtheorem{theorem}{Theorem}[section]
\newtheorem{lemma}[theorem]{Lemma}
\newtheorem{definition}[theorem]{Definition}

\newtheorem{corollary}[theorem]{Corollary}

\newtheorem{observation}[theorem]{Observation}

\newtheorem{remark}[theorem]{Remark}
\newtheorem{claim}[theorem]{Claim}
\newtheorem{example}[theorem]{Example}

\newcommand{\bi}{\boldsymbol{i}}

\newcommand{\F}{\mathcal{F}}

\newcommand{\wh}{\widehat}
\newcommand{\wt}{\widetilde}

\newcommand{\eps}{\varepsilon}
\newcommand{\Ot}{\widetilde{O}}

\newcommand{\norm}[1]{\|#1\|}
\newcommand{\abs}[1]{|#1|}

\newcommand{\Tr}{\mathsf{Tr}}

\newcommand{\supp}{\mathsf{supp}}

\DeclareMathOperator*{\argmin}{arg\,min}
\DeclareMathOperator*{\E}{\mathbb{E}}

\newcommand{\R}{\mathbb{R}}
\newcommand{\C}{\mathbb{C}}

\newcommand{\FF}{\mathcal{F}}
\renewcommand{\i}{\mathbf{i}}
\newcommand{\midd}{\mathsf{mid}}

\DeclareMathOperator*{\argmax}{arg\,max}

\newenvironment{proofof}[1]{\bigskip \noindent {\it Proof of #1.}\quad }
{\qed\par\vskip 4mm\par}

\usepackage{etoolbox}

\makeatletter

\newcommand{\define}[4][ignore]{%
  \ifstrequal{#1}{ignore}{}{
  \@namedef{thmtitle@#2}{#1}}%
  \@namedef{thm@#2}{#4}%
  \@namedef{thmtypen@#2}{lemma}%
  \newtheorem{thmtype@#2}[theorem]{#3}%
  \newtheorem*{thmtypealt@#2}{#3~\ref{#2}}%
}

\newcommand{\state}[1]{%
  \@namedef{curthm}{#1}
  \@ifundefined{thmtitle@#1}{
  \begin{thmtype@#1}
    }{
  \begin{thmtype@#1}[\@nameuse{thmtitle@#1}]
  }
    \label{#1}
    \@nameuse{thm@#1}
  \end{thmtype@#1}
  \@ifundefined{thmdone@#1}{
  \@namedef{thmdone@#1}{stated}%
  }{}
}

\newcommand{\restate}[1]{%
  \@namedef{curthm}{#1}
  \@ifundefined{thmtitle@#1}{
    \begin{thmtypealt@#1}
    }{
  \begin{thmtypealt@#1}[\@nameuse{thmtitle@#1}]
  }
    \@nameuse{thm@#1}
  \end{thmtypealt@#1}
  \@ifundefined{thmdone@#1}{
  \@namedef{thmdone@#1}{stated}%
  }{}
}

\newcommand{\thmlabel}[1]{
  \@ifundefined{thmdone@\@nameuse{curthm}}{\label{#1}
    }{\tag*{\eqref{#1}}}
}
\makeatother

\begin{document}

\title{Active Regression via Linear-Sample Sparsification}
\author{   Xue Chen\thanks{Supported by NSF Grant CCF-1526952 and a Simons Investigator Award (\#409864, David Zuckerman), part of this work was done while the author was a student in University of Texas at Austin and visiting the Simons Institute for the Theory of Computing.}\\
  \texttt{xue.chen1@northwestern.edu}\\
  Northwestern University
 \and   Eric Price\thanks{This work was done in part while the author was visiting the Simons Institute for the Theory of Computing.}\\
  \texttt{ecprice@cs.utexas.edu}\\
  The University of Texas at Austin
}

\begin{titlepage}
  
\maketitle

\begin{abstract}
  We present an approach that improves the sample complexity for a
  variety of curve fitting problems, including active learning for
  linear regression, polynomial regression, and continuous sparse
  Fourier transforms.  In the active linear regression problem, one
  would like to estimate the least squares solution $\beta^*$
  minimizing $\norm{X\beta - y}_2$ given the entire unlabeled dataset
  $X \in \R^{n \times d}$ but only observing a small number of labels
  $y_i$.  We show that $O(d)$ labels
  suffice to find a constant factor approximation $\wt{\beta}$:
  \[
    \E[\norm{X\wt{\beta} - y}_2^2] \leq 2 \E[\norm{X \beta^* - y}_2^2].
  \]
  This improves on the best previous result of $O(d \log d)$ from
  leverage score sampling.  We also present results for the
  \emph{inductive} setting, showing when $\wt{\beta}$ will generalize
  to fresh samples; these apply to continuous settings such as
  polynomial regression.  Finally, we show how the techniques yield
  improved results for the non-linear sparse Fourier transform
  setting.
\end{abstract}

\thispagestyle{empty}
\end{titlepage}

\section{Introduction}

We consider the query complexity of recovering a signal $f(x)$ in a given family $\FF$ from noisy observations. This problem takes many forms depending on the family $\FF$, the access model, the
desired approximation norms, and the measurement distribution. In this work, we consider the $\ell_2$ norm and use $D$ to denote the distribution on the domain of $\FF$ measuring the distance between different functions, which is not necessarily known to our algorithms.



Our main results are sampling mechanisms that improve the query complexity and guarantees for two specific families of functions --- linear families and continuous sparse Fourier transforms.

\paragraph{Active Linear Regression on a Finite Domain.}  We start with
the classical problem of linear regression, which involves a matrix
$X \in \R^{n \times d}$ representing $n$ points with $d$ features, and
a vector $y \in \R^n$ representing the labels associated with those
points.  The least squares solution is
\[
  \beta^* := \argmin \norm{X\beta - y}_2^2.
\]
In one \emph{active learning} setting, we receive the entire matrix
$X$ but not the entire set of labels $y$ (e.g., receiving any given
$y_i$ requires paying someone to label it).  Instead, we can pick a
small subset $S \subseteq [n]$ of size $m \ll n$, observe $y_S$, and
must output $\wt{\beta}$ that accurately predicts the entire set of
labels $y$.  In particular, one would like
\[
  \norm{X\wt{\beta} - y}_2^2 \leq (1 + \eps) \norm{X \beta^* - y}_2^2
\]
or
\begin{align}\label{eq:mateq}
  \norm{X\wt{\beta} - X\beta^*}_2^2 \leq \eps \norm{X \beta^* - y}_2^2.
\end{align}
This is known as the ``transductive'' setting, because it only
considers the prediction error on the given set of points $X$; in the
next section we will consider the ``inductive'' setting where the
sample points $X_i$ are drawn from an unknown distribution and we care
about the generalization to fresh points.

The simplest approach to achieve~\eqref{eq:mateq} would be to sample
$S$ uniformly over $[n]$.  However, depending on the matrix, the
resulting query complexity $m$ can be very large -- for example, if
one row is orthogonal to all the others, it must be sampled to
succeed, making $m \geq n$ for this approach.

A long line of research has studied how to improve the query
complexity by adopting some form of importance sampling.  Most
notably, sampling proportional to the leverage scores of the matrix
$X$ improves the sample complexity to $O(d \log d + d/\eps)$ (see,
e.g.,~\cite{mahoney2011randomized}).

In this work, we give an algorithm that improves this to $O(d/\eps)$,
which we show is optimal.  The $O(d\log d)$ term in leverage score
sampling comes from the coupon-collector problem, which is inherent to
any i.i.d.\ sampling procedure.  By using the randomized linear-sample
spectral sparsification algorithm of Lee and Sun~\cite{LeeSun}, we can
avoid this term.  Note that not every linear spectral sparsifier would
suffice for our purposes: deterministic algorithms
like~\cite{batson2012twice} cannot achieve~\eqref{eq:mateq} for
$m \ll n$ (see \cite{boutsidis2013near}).  We exploit the particular behavior of~\cite{LeeSun} to
bound the expected noise in each step.

\define{thm:BSS}{Theorem}{%
  Given any $n \times d$ matrix $X$ and vector
  $\vec{y} \in \mathbb{R}^n$, let
  $\beta^*=\underset{\beta \in \mathbb{R}^d}{\arg\min} \|X \beta -
  \vec{y} \|_2^2$. For any $\eps < 1$, we present an efficient
  randomized algorithm that looks at $X$ and produces a diagonal
  matrix $W_S$ with support $S \subseteq [n]$ of size
  $\abs{S} \leq O(d/\eps)$, such that
  \[
    \wt{\beta} := \underset{\beta}{\arg\min} \|W_S X \cdot \beta - W_S \cdot \vec{y} \|_2
  \]
  satisfies
  \[
    \E \left[ \| X \cdot \wt{\beta} - X \cdot \beta^* \|_2^2 \right] \le \eps \cdot \|X \cdot \beta^* - \vec{y} \|_2^2.
  \]
  In particular, this implies
  $\|X \cdot \wt{\beta} - \vec{y} \|_2 \le \big( 1+O(\eps) \big) \cdot
  \|X \cdot \beta^* - \vec{y} \|_2$ with 99\% probability.  }

\state{thm:BSS}

At the same time, we provide a theoretic information lower bound
$m=\Omega(d/\eps)$ matching the query complexity up to a constant
factor, when $\vec{y}$ is $X\beta^*$ plus i.i.d.\ Gaussian noise.

\paragraph{Generalization for Active Linear Regression.}  We now
consider the inductive setting, where the $(x, y)$ pairs come from
some unknown distribution over $\R^d \times \R$.  As in the
transductive setting, we see $n$ unlabeled points
$X \in \R^{n \times d}$, choose a subset $S \subset [n]$ of size $m$
to receive the labels $y_S$ for, and output $\wt{\beta}$.  However,
the guarantee we want is now with respect to the unknown distribution:
for
\[
  \beta^* := \argmin \E_{x,y} [(x^T\beta - y)^2],
\]
we would like
\[
  \E_{x,y}[(x^T\wt{\beta} - y)^2] \leq (1 + \eps) \E_{x,y}[(x^T\beta^* - y)^2]
\]
or
\[
  \E_{x}[(x^T\wt{\beta} - x^T\beta^*)^2] \leq \eps \E_{x,y}[(x^T\beta^* - y)^2].
\]

In this inductive setting, there are now two parameters we would like
to optimize: the number of labels $m$ and the number of unlabeled
points $n$.  Our main result shows that there is no significant
tradeoff between the two: as soon $n$ is large enough that the empirical risk minimizer (ERM) for
a fully labeled dataset would generalize well, one can apply
Theorem~\ref{thm:BSS} to only label $O(d/\eps)$ points; and even with
an infinitely large unlabeled data set, one would still require
$\Theta(d/\eps)$ labels.

But how many unlabeled points do we need for the ERM to generalize?
To study this, we consider a change in notation that makes it more
natural to consider problems like polynomial regression.  In
polynomial regression, suppose that $y \approx p(x)$, for $p$ a degree
$d-1$ polynomial and $x$ on $[-1, 1]$.  This is just a change in
notation, since one could express $p(x)$ as
$(1, x, ..., x^{d-1})^T \beta$ for some $\beta$.  How many
observations $y_i = p(x_i) + g(x_i)$ do we need to learn the
polynomial for any noise function $g$, in the sense that
\[
  \E_{x \in [-1, 1]} [(\wt{p}(x) - p(x))^2] \leq O(1) \cdot \E [g(x)^2]?
\]
If we sample $x$ uniformly on $[-1, 1]$, then $\Omega(d^2)$ samples are
necessary; if we sample $x$ proportional to the Chebyshev weight 
$\frac{1}{\sqrt{1 - x^2}}$, then $O(d \log d)$ samples suffice (this
is effectively leverage score sampling); and if we sample $x$ more
carefully, we can bring this down to $O(d)$~\cite{CDL13,CKPS}.  This
work shows how to perform similarly for \emph{any} linear family of
functions, including multivariate polynomials.  We also extend the
result to \emph{unknown} distributions on $x$.

In the model we consider, then, $x$ is drawn from an unknown
distribution $D$ over an arbitrary domain $G$, and $y = y(x_i)$ is
sampled from another unknown distribution conditioned on $x_i$.  We
are given a dimension-$d$ linear family $\FF$ of functions
$f: G \to \C$.  Given $n$ samples $x_i$, we can pick $m$ of the $y_i$
to observe, and would like to output a hypothesis $\wt{f} \in \FF$
that is predictive on fresh samples:
 \begin{equation}\label{eq:approximation}
 \|\wt{f}-f^*\|_D^2 \overset{\text{def}}{:=}\E_{x \sim D}[|\wt{f}(x)-f^*(x)|^2] \le \eps \cdot \E_{x, y}[|y - f^*(x)|^2]
\end{equation}
where $f^* \in \FF$ minimizes that RHS.  The polynomial regression
problem is when $\FF$ is the set of degree-$(d-1)$ polynomials in the
limit as $n \to \infty$, since we know the distribution $D$ and can
query any point in it.


\define{thm:unknown_dist}{Theorem}{ 

  Let $\FF$ be a linear family of functions from a domain $G$ to
  $\mathbb{C}$ with dimension $d$, and consider any (unknown)
  distribution on $(x, y)$ over $G \times \C$.  Let $D$ be the
  marginal distribution over $x$, and suppose it has a bounded
  ``condition number''
  \begin{align}\label{eq:defineK}
    K := \underset{h \in \FF: h \neq 0}{\sup} \frac{\sup_{x \in G}|h(x)|^2}{\|h\|_D^2}.
  \end{align}
  Let $f^* \in \FF$ minimize $\E[\abs{f(x) - y}^2]$. For any error $\eps>0$ and failure probability $\delta$, there
  exists an efficient algorithm that takes
  $m:=O(K \log \frac{d}{\delta} + \frac{K}{\eps})$ unlabeled samples from $D$ and
  requests $O(\frac{d}{\eps})$ labels to output $\wt{f}$
  satisfying 
\[
\underset{x \sim D}{\E}[|\tilde{f}(x)-f(x)|^2] \le \eps \cdot \underset{(x,y)\sim (D,Y)}{\E}[|y-f(x)|^2],
\]
conditioned on an $(1-\delta)$-probability random event $\mathcal{E}$:
\begin{equation}\label{eq:def_event}
    \underset{\text{unlabeled samples} x_i}{\E} |h(x_i)|^2 = (1 \pm 1/4) \cdot \|h\|_D^2 \text{ for any }h \in \mathcal{F}.
\end{equation}
}
\state{thm:unknown_dist}


A few points are in order.  First, the random event $\mathcal{E}$ in Theorem~\ref{thm:unknown_dist} indicates these unlabeled samples provide a constant-error spectral approximation of 
$\mathcal{F}$, which holds with probability $1-\delta$ from the matrix Chernoff bound \cite{Tropp} (see Lemma~\ref{lem:agnostic_learning_single_distribution} for a formal proof). This is necessary because the bias of the spectral approximation will increase the error $\E_{x \sim D}[|\wt{f}(x)-f^*(x)|^2]$. It is different from Theorem~\ref{thm:BSS} because algorithms in Theorem~\ref{thm:unknown_dist} have a fixed amount of unlabeled samples.
Secondly, notice that if we merely want to
optimize the number of labels, it is possible to take infinite number
of samples from $D$ to learn it and then query whatever desired labels
on $x \in \supp(D)$. This is identical to the query access model,
where $\Theta(d/\eps)$ queries is necessary and sufficient from
Theorem~\ref{thm:BSS}. On the other hand, if we focus on unlabeled
sample complexity, a natural solution is to query every sample point
and calculating the ERM $\wt{f}$; one can show that this takes
$\Theta(K \log d + K /\eps)$ samples~\cite{CDL13}. Thus both the
unlabeled and labeled sample complexity of our algorithm are optimal
up to a constant factor.

Finally, in settings with a ``true'' signal $f(x)$ one may want
$\wt{f} \approx f$ rather than $\wt{f} \approx f^*$.  Such a result
follows directly from the Pythagorean theorem, although (if the noise
is biased, so $f^* \neq f$) the approximation becomes $(1 + \eps)$
rather than $\eps$:

\begin{corollary}
  Suppose that $y(x) = f(x) + g(x)$, where $f \in \FF$ is the ``true''
  signal and $g$ is arbitrary and possibly randomized ``noise''.  Then
  in the setting of Theorem~\ref{thm:unknown_dist}, with
  $\norm{\cdot}_D$ defined as in~\eqref{eq:approximation}, conditioned on the random event $\mathcal{E}$,
  \begin{enumerate}
  \item $\E[ \|\wt{f}-f\|^2_D] \le \eps \cdot \E[\|g\|_D^2]$, if each
    $g(x)$ is a random variable with $\E_{x,g}[g(x)]=0$.
  \item Otherwise, $\|\wt{f}-f\|_D \le (1+O(\eps)) \cdot \|g\|_D$ with probability 0.99.
  \end{enumerate}
\end{corollary}

To make the result concrete, we present the following implication:

 \begin{example}
   Consider fitting $n$-variate degree-$d$ polynomials on $[-1, 1]^n$.
   There are $\binom{n+d}{d}$ monomials in the family, so
   Theorem~\ref{thm:unknown_dist} shows that querying
   $O(\binom{n+d}{d})$ points can achieve a constant-factor
   approximation to the optimal polynomial.  By contrast, uniform
   sampling would work well for low $d$, but loses a $\text{poly}(d)$
   factor; Chebyshev sampling would work well for low $n$, but loses a
   $2^{O(n)}$ factor; leverage score sampling would lose a
   $\log \binom{n+d}{d}$ factor.
 \end{example}

 \paragraph{Continuous Sparse Fourier transform.} Next we study
 sampling methods for learning a non-linear family: $k$-Fourier-sparse
 signals in the continuous domain. We consider the family of
 bandlimited $k$-Fourier-sparse signals
\begin{equation}\label{eq:def_k_FT}
\FF=\left\{ f(x)=\sum_{j=1}^k v_j \cdot e^{2 \pi \bi f_j x} \bigg| f_j \in \mathbb{R} \cap [-F,F], v_j \in \mathbb{C} \right\}
\end{equation}
over the domain $D$ uniform on $[-1, 1]$.

Because the frequencies $f_j$ can be any real number in $[-F,F]$, this
family is not well conditioned.  If all $f_j \to 0$, a Taylor
approximation shows that one can arbitrarily approximate any degree
$(k-1)$ polynomial; hence $K$ in \eqref{eq:defineK} is at least
$\Theta(k^2)$.

To improve the sample complexity of learning $\FF$, we apply
importance sampling for it by biasing $x \in [-1,1]$ proportional to
the largest variance at each point:
$ \underset{f \in \FF}{\sup} \frac{|f(x)|^2}{ \|f\|_D^2}.  $ This is a
natural extension of leverage score sampling, since it matches the
leverage score distribution when $\FF$ is linear. Our main
contribution is a simple upper bound that closely approximates the
importance sampling weight for $k$-Fourier-sparse signals at every
point $x \in (-1,1)$.

\define{lem:bound_k_sparse_FT_x_middle}{Theorem}{
For  any $x \in (-1,1)$,
\[
\underset{f \in \FF}{\sup} \frac{|f(x)|^2}{\|f\|_D^2} \lesssim \frac{k \log k}{1-|x|}.
\]
} \state{lem:bound_k_sparse_FT_x_middle} Combining this with the
condition number bound $K=\tilde{O}(k^4)$ in \cite{CKPS}, this gives
an explicit sampling distribution with a ``reweighted'' condition
number (as defined in Section~\ref{sec:overview}) of $O(k \log^2 k)$;
this is almost tight, since $k$ is known to be necessary. We show the
weight density in Figure \ref{fig:fourierweight}.

\begin{figure}[H]
\center
      \includegraphics[height=0.3\textheight, width=0.6\textwidth]{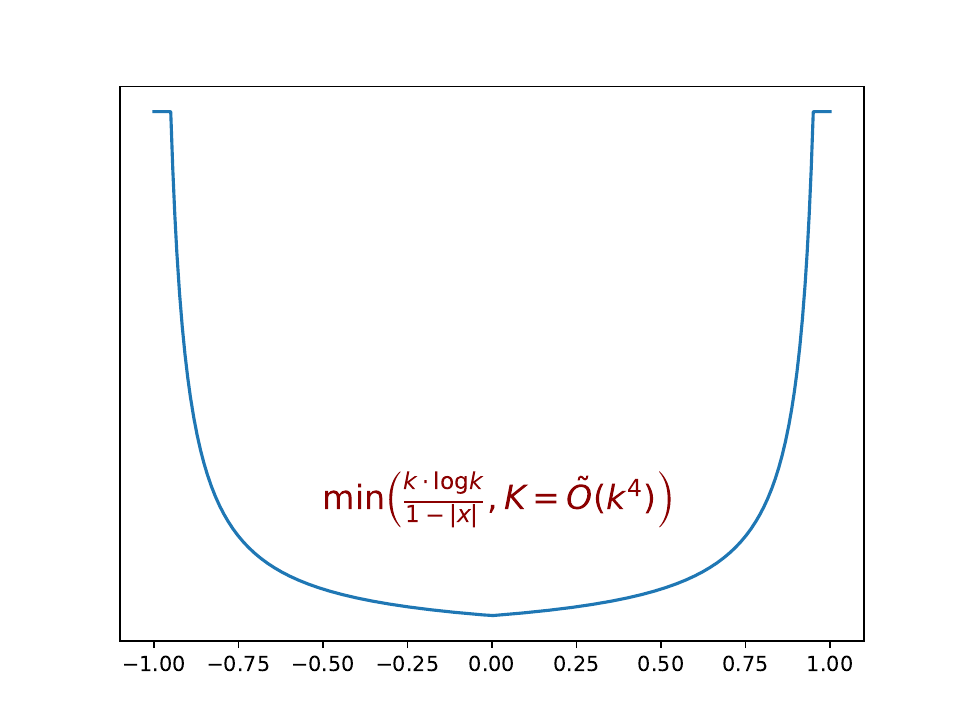}
      \caption{Explicit weights for $k$-Fourier-sparse signals} \label{fig:fourierweight}
\end{figure}

The reweighted condition number indicates that $m=\tilde{O}(k)$
suffices for the empirical estimation of $\|f\|_D$ for any fixed
$f \in \FF$.  We show that this implies that
$m=\tilde{O}(k^4 + k^2 \cdot \log F)$ guarantees the empirical
estimation of \emph{all} $f \in \FF$, so that the ERM
$\wt{f} \approx f$ (The extra loss of $m$ is due to the infinitely
many possible frequencies in this family).  This is a much better
polynomial in $k$ than was previously known to be possible for the
problem~\cite{CKPS}.  We believe that this sampling approach directly
translates to improvements in the polynomial time recovery algorithm
of~\cite{CKPS}, but that algorithm is quite complicated so we leave
this for future work.

\subsection{Related Work}
\paragraph{Linear regression.}  A large body of work considers ways to
subsample linear regression problems so the solution $\wt{\beta}$ to
the subsampled problem approximates the overall solution $\beta^*$.
The most common such method is leverage score sampling, which achieves
the guarantee of~Theorem~\ref{thm:BSS} with $O(d \log d + d/\eps)$
samples~\cite{drineas2008relative,M10,mahoney2011randomized,W14}.

Several approaches have attempted to go beyond this $O(d \log d)$
sample complexity. Both \cite{boutsidis2013near} and
\cite{song2019relative} apply the deterministic linear-sample spectral
sparsification of \cite{batson2012twice} to the matrix
$n \times (d+1)$ matrix $[X | y]$, to find a size $O(d/\eps)$ set that
would suffice for Theorem~\ref{thm:BSS}.  However, this procedure
requires knowing the entirety of $y$ to find $S$ \cite{boutsidis2013near}, so it does not help
for active learning. \cite{allen2017near} showed how such a procedure
can additionally have a number of extra properties, such as that each
sample has equal weight.  However, all these results involve
deterministic sampling procedures, so cannot tolerate adversarial
noise.

Another line of research on minimizing the query complexity of linear
regression is volume sampling, which samples a set of points
proportional to the volume of the parallelepiped spanned by these points. Recently,
\cite{derezinski2017unbiased} showed that exactly $d$ points chosen
from volume sampling can achieve the guarantee of
Theorem~\ref{thm:BSS}, except with an approximation ratio $d+1$ rather than $1+O(\epsilon)$.
In a subsequent work, \cite{derezinski2018tail} showed that standard
volume sampling would need $\Omega(K)$ samples to achieve any
constant approximation ratio, but that a variant of volume
sampling can match leverage score sampling with $O(d\log d + d/\eps)$
samples.

\vspace{0.1in}
\begin{minipage}[t]{0.9\linewidth}
  \renewcommand{\arraystretch}{1.3}
  \begin{tabular}{| l | l | l |}
      \hline
    & \# queries & Remark \\
    \hline
    Uniform sampling \hfill \cite{CDL13,hsu2016loss} & $O(K \log d)$ &   \\
    \hline
    Leverage score sampling\hfill\cite{mahoney2011randomized} & $\Theta(d \log d)$ &  \\
    \hline
    Boutsidis et al.\hfill\cite{boutsidis2013near} & $\abs{\supp(D)}$ & only needs $O(d)$ points in \eqref{eq:approximation}\\
    \hline
    Volume sampling\hfill \cite{derezinski2017unbiased} & $d$ & for $\eps = d+1$\\
    \hfill\cite{derezinski2018tail} & $\Omega\big( K \big)$  & for $\eps = 0.5$\\
	  \hline
    Rescaled volume sampling \hfill \cite{derezinski2018tail} & $O(d \log d)$ &\\
    \hline
    This work & $O(d)$ & \\
    \hline
  \end{tabular}
  \captionof{table}{Summary on the sample complexity of learning
    linear families, for $\eps = \Theta(1)$ unless otherwise
    specified.  Note that $K \geq d$.}
\end{minipage}
\vspace{0.1in}

For active regression, \cite{sabato2014active} provide an algorithm of
$O(d \log d)^{5/4}$ labels to achieve the desired guarantee of ERM,
while they do not give an explicit bound on the number of unlabeled
points in the algorithm. \cite{chaudhuri2015convergence} propose a
different approach assuming additional structure for the distribution
$D$ and knowledge about the noise $g$, allowing stronger results than
are possible in our setting.

One linear family of particular interest is univariate polynomials of
degree $d$ with the uniform distribution over $[-1,1]$. \cite{CDL13}
show that $O(K \log d)$ samples suffices for \eqref{eq:approximation}
for any linear family. In particular, they prove $m=O(d \log d)$
samples generated from the Chebyshev weight are sufficient, because it
is the limit of the leverage scores of univarite
polynomials. \cite{CKPS} avoids the extra loss $\log d$ by generating
every point $x_i$ using a distinct distribution: it partitions the
Chebyshev weight into $O(d)$ intervals of equal summations and sample
one point from each interval. However, this partition may not exist
for arbitrary linear families and distributions.

\paragraph{Sparse Fourier transform.} There is a long line of research
on sparse Fourier transform in the continuous setting, e.g., Prony's
method from 1795, Hilbert's inequality by \cite{HilbertMV} and Matrix
Pencil method \cite{BM86,Moitra15} to name a few. At the same time,
less is known about the worst case guarantees without any assumption
on separation between the frequencies; this depends on the condition
number $K$, which is between $k^2$ and $\tilde{O}(k^4)$ as noted
above~\cite{CKPS}. We note in passing the bound on $K$ and
Theorem~\ref{lem:bound_k_sparse_FT_x_middle} is analogous to Markov
Brothers' inequality and the Bernstein inequality for univariate
polynomials.

A number of works have studied importance sampling for sparse recovery
and sparse Fourier transforms. \cite{rauhut2012sparse} considered the
case where $\FF$ is sparse in a well-behaved orthonormal basis such as
polynomials sparse in Legendre basis using the Chebyshev
distribution. We refer to the survey \cite{ward2015importance} for a
detailed discussion. Recently, \cite{AKMMVZ17} give a study about
kernel ridge scores for signals with known Fourier transform
structures such as the Gaussian kernel in multi-dimension. However,
the weight shown in \cite{AKMMVZ17} is not close to optimal for
multi-dimension, while our weight is almost tight.

\paragraph{Organization.} We introduce our approaches and ``well-balanced'' procedures and outline the proofs of our results in Section~\ref{sec:overview}. After introducing notation and tools in Section~\ref{sec:preliminaries}, we prove a ``well-balanced'' procedure guarantees \eqref{eq:approximation} with high probability. Then we show the randomized spectral sparsification of \cite{LeeSun} is ``well-balanced'' in Section~\ref{sec:BSS}. For completeness, we analyze the number of samples generated by one distribution in Section~\ref{sec:single_batch}. Next we combine the results of the previous two sections to prove our results about active learning in Section~\ref{sec:act_learning}. We show information lower bound on the sample complexity in Section~\ref{sec:lower_bound}. Finally, we prove our results about sparse Fourier transform in Section~\ref{sec:k_sparse_FT}.


\section{Proof Overview}\label{sec:overview}

We present our proof sketch in the notation of
Theorem~\ref{thm:unknown_dist}.  Consider observations of the form
$y(x)=f(x)+g(x)$ for $f$ in a (not yet necessarily linear) family $\FF$
and $g$ an arbitrary, possibly random function.

\paragraph{Improved conditioning by better sampling.} We start with the noiseless case of
$g = 0$ in the query access model, and consider the problem of
estimating $\|y\|_D^2=\|f\|_D^2$ with high probability.  If we sample
points $x_i \sim D'$ for some distribution $D'$, then we can estimate
$\norm{f}_D^2$ as the empirical norm
\begin{align}\label{eq:weightedestimate}
  \frac{1}{m}\sum_{i=1}^m \frac{D(x_i)}{D'(x_i)} \abs{f(x_i)}^2
\end{align}
which has the correct expectation.  To show the expectation
concentrates, we should bound the maximum value of the summand, which
we define to be the ``reweighted'' condition number
\[
  K_{D'}=\underset{x}{\sup} \bigg\{ \underset{f \in \FF}{\sup} \big\{
  \frac{D(x)}{D'(x)} \cdot \frac{|f(x)|^2}{\|f\|_D^2} \big\}\bigg\}.
\]
We define $D_{\FF}$ to minimize this quantity, by making the inner
term the same for every $x$.  Namely, we pick
\begin{equation}\label{def:kappa}
D_{\FF}(x)= \frac{1}{\kappa} D(x)  \cdot \underset{f \in \FF}{\sup} \frac{|f(x)|^2}{\|f\|_D^2} \qquad \text{for } \kappa=\underset{x \sim D}{\E}\left[\underset{f \in \FF}{\sup} \frac{|f(x)|^2}{\|f\|_D^2} \right].
\end{equation}
This shows that by sampling from $D_{\FF}$ rather than $D$, the condition
number of our estimate~\eqref{eq:weightedestimate} improves from
$K = \underset{x \in \supp(D)}{\sup} \big\{ \underset{f \in \FF}{\sup}
\frac{|f(x)|^2}{\|f\|^2_D} \big\}$ to
 $\kappa = \underset{x \sim D}{\E}\left[\underset{f \in \FF}{\sup}
  \frac{|f(x)|^2}{\|f\|_D^2} \right]$.

From the Chernoff bound,
$O(\frac{\kappa \cdot \log \frac{1}{\delta}}{\eps^2})$ samples
from $D_{\FF}$ let us estimate $\|f\|_D^2$ to within accuracy
$1 \pm \eps$ with probability $1-\delta$ for any fixed function
$f \in \FF$.  To be able to estimate \emph{every} $f \in \FF$, a basic
solution would be to apply a union bound over an $\eps$-net of $\FF$.

Linear function families $\FF$ let us improve the result in two ways.
First, we observe that $\kappa = d$ for any dimension $d$ linear
function space; in fact, $D_\FF$ is the leverage score sampling
distribution.  Second, we can replace the union bound by a matrix
concentration bound, showing that
$O(\frac{d \log \frac{d}{\delta}}{\eps^2})$ samples from $D_{\FF}$
suffice to estimate $\|f\|_D^2$ to within $1 \pm \eps$ for all
$f \in \FF$ with probability $1-\delta$. However, this approach needs
$\Omega(d \log d)$ samples due to a coupon-collector argument, because
it only samples points from one distribution $D_\FF$.



\paragraph{The effect of noise.} The spectral sparsifier given by \cite{batson2012twice} could
replace the matrix concentration bound above, and estimate
$\norm{f}_D^2$ for every $f \in \FF$ with only $O(d)$ samples.
The issue with this is that it would not be robust against adversarial
noise, because the sample points $x_i$ are deterministic. Now we consider our actual problem, which is to estimate $f$
from $y = f + g$ for nonzero noise $g$ of
$\E[\norm{g}_D^2] = \sigma^2$.  We need our sample points to both be
sampled non-independently (to avoid coupon-collector issues) but still
fairly randomly (so adversarial noise cannot predict it).  A natural
approach is to design a sequence of distributions
$D_1,\cdots,D_m$ ($m$ is not necessarily fixed) then sample
$x_i \sim D_i$ and assign a weight $w_i$ for $x_i$, where $D_{i+1}$
could depend on the previous points $x_1,\cdots,x_i$.

Given samples $(x_1,\cdots,x_m)$ with weights $(w_1,\cdots,w_m)$, the
empirical risk minimizer is the function $\wt{f} \in \FF$ closest to $y$ under
the empirical norm $\underset{i\in [m]}{\sum} w_i \cdot |f(x_i)|^2$.  When
$\F$ is a linear family, the solution $\wt{f}$ is a linear projection,
so it acts on $f$ and $g$ independently.  If the empirical norm is a good
estimator for $\F$, the projection of $f \in \FF$ into the linear subspace $\F$ equals $f$.  Hence the error $\wt{f} - f$ is the projection of $g$
onto $\F$ under the empirical norm.

First, suppose that $g$ is orthogonal to $\F$ under the true norm
$\norm{\cdot}_D$---for instance, if $g(x)$ is an independent mean-zero
random variable for each $x$.  In this case, the expected value of the
projection of $g$ is zero.  At the same time, we can bound the
variance of the projection of a single random sample of $g$ drawn from
$D_i$ by the condition number $K_{D_i} \cdot \sigma^2$.  Ideally each
$K_{D_i}$ would be $O(d)$, but we do not know how to produce such
distributions while still getting linear sample spectral
sparsification.  Therefore we use a coefficient $\alpha_i$ to control
every $K_{D_i}$, and set $w_i=\alpha_i \cdot \frac{D(x_i)}{D_i(x_i)}$
instead of $\frac{D(x_i)}{m D_i(x_i)}$.  The result is that---if
$\sum_i \alpha_i=O(1)$---the projection of the noise has variance
$O\big(\underset{i \in [m]}{\max} \{\alpha_i K_{D_i}\}\big) \cdot
\sigma^2$. This motivates our definition of ``well-balanced''
sampling procedures:

\define{def:procedure_agnostic_learning}{Definition}{
  Given a linear family $\FF$ and underlying distribution $D$, let $P$
  be a random sampling procedure that terminates in $m$ iterations
  ($m$ is not necessarily fixed) and provides a coefficient $\alpha_i$
  and a distribution $D_i$ to sample $x_i \sim D_i$ in every iteration
  $i \in [m]$.

We say $P$ is an $\eps$-\emph{well-balanced sampling procedure} if it satisfies the following two properties:
\begin{enumerate}
\item With probability $0.9$, for weight $w_i=\alpha_i \cdot \frac{D(x_i)}{D_i(x_i)}$ of each $i \in [m]$, $$
\sum_{i=1}^m w_i \cdot |h(x_i)|^2 \in \left[\frac{3}{4},\frac{5}{4}\right] \cdot \|h\|_D^2 \quad \forall h \in \FF.
$$
Equivalently (as shown in Lemma~\ref{lem:operator_estimation} in Section~\ref{sec:proof_AL_procedure}), given any orthonormal basis $v_1,\ldots,v_d$ of $\FF$ under $D$, the matrix $A(i,j)=\sqrt{w_i} \cdot v_j(x_i) \in \mathbb{C}^{m \times d}$ has $\lambda(A^* A) \in [\frac{3}{4},\frac{5}{4}]$.

\item The coefficients always have $\sum_i \alpha_i \le \frac{5}{4}$ and $\alpha_i \cdot K_{D_i} \le \epsilon/2$. 
\end{enumerate}
}

\state{def:procedure_agnostic_learning}


Intuitively, the first property says that the sampling procedure
preserves the signal, and the second property says that the recovery
algorithm does not blow up the noise on average.  For such a sampling
procedure we consider the ERM from its execution as follows.

Our definition of a well-balanced sampling procedure allows property 1
to fail 10\% of the time, but our algorithm will only perform well in
expectation when property 1 is satisfied.  Therefore we rerun the
sampling procedure until it has a ``good'' execution that satisfies
property 1.

\begin{definition}\label{def:weightederm}
Given a \emph{well-balanced sampling procedure} $P$, we say one execution of $P$ is \emph{good} only if the samples $x_i$ with weights $w_i=\alpha_i \cdot \frac{D(x_i)}{D_i(x_i)}$ satisfy the first property in Definition~\ref{def:procedure_agnostic_learning}, which can be checked efficiently by calculating $\lambda(A^* A)$.

  Given a joint distribution  $(D, Y)$ and an execution of a \emph{well-balanced sampling procedure} $P$ with $x_i \sim D_i$ and $w_i=\alpha_i \cdot \frac{D(x_i)}{D_i(x_i)}$ of each $i \in [m]$, let the weighted ERM of this execution be $\wt{f}=\underset{h \in \FF}{\argmin} \left\{ \sum_{i=1}^m w_i
      \cdot |h(x_i)-y_i|^2
    \right\}$ by querying $y_i \sim (Y|x_i)$ for each point $x_i$ .
\end{definition}

In Section~\ref{sec:agnostic_learning} we prove that $\wt{f}$ satisfies the desired guarantee, which implies Theorem~\ref{thm:BSS}.

\define{thm:guarantee_AL_procedure}{Theorem}{ Given a linear family
  $\FF$, joint distribution $(D,Y)$, and $\eps>0$, let $P$ be an
  $\eps$-well-balanced sampling procedure for $\FF$ and $D$,
  and let
  $f=\underset{h \in \FF}{\arg\min} \underset{(x,y)\sim
    (D,Y)}{\E}[|y-h(x)|^2]$ be the true risk minimizer.  Then the
  weighted ERM $\wt{f}$ of a good execution of $P$ satisfies
  \[
    \|f-\wt{f}\|_D^2 \le \epsilon \cdot \underset{(x,y)\sim (D,Y)}{\E}[|y-f(x)|^2] \textit{ in expectation}.
  \]
}
\state{thm:guarantee_AL_procedure}



For a noise function $g$ not orthogonal to $\FF$ in expectation, let
$g^{\bot}$ and $g^{\parallel}$ denote the decomposition of $g$ where
$g^{\bot}$ is the orthogonal part and
$g^{\parallel}=g-g^{\bot} \in \FF$. The above theorem indicates
$\|\wt{f}-f\|_D \le \|g^{\parallel}\|_D + \sqrt{\eps} \cdot \|g^{\bot}\|_D$, which
gives $(1+\eps)\|g\|_D$-closeness via the Pythagorean theorem. This result appears in Corollary~\ref{cor:specific_noise_AL_procedure} of Section~\ref{sec:agnostic_learning}.


\paragraph{Well-balanced sampling procedures.} We observe that two
standard sampling procedures are well-balanced, so they yield agnostic
recovery guarantees by Theorem~\ref{thm:guarantee_AL_procedure}.  The
simplest approach is to set each $D_i$ to be a fixed distribution $D'$
and $\alpha_i = 1/m$ for all $i$.  For
$m=O(K_{D'} \log d + K_{D'}/\eps)$, this gives an $\eps$-well-balanced
sampling procedure. These results appear in
Section~\ref{sec:single_batch}.


We get a stronger result of $m = O(d/\eps)$ using the randomized BSS
algorithm from~\cite{LeeSun}.  The~\cite{LeeSun} algorithm
iteratively chooses points $x_i$ from distributions $D_i$.  A term
considered in their analysis---the largest increment of
eigenvalues---is equivalent to our $K_{D_i}$.  By looking at the
potential functions in their proof, we can extract coefficients
$\alpha_i$ bounding $\alpha_i K_{D_i}$ in our setting. This lets us
show that the algorithm is a well-balanced sampling procedure; we do
so in Section~\ref{sec:BSS}.

\paragraph{Active learning.} Next we consider the active learning
setting, where we don't know the distribution $D$ and only receive
samples $x_i \sim D$, but can choose which $x_i$ receive labels $y_i$.
Let $K$ be the condition number of the linear family $\FF$.  Our algorithms
uses $n=O(K \log d + \frac{K}{\eps})$ unlabeled samples and
$m = O(\frac{d}{\eps})$ labeled samples, and achieves the
same guarantees as in the query access model.

For simplicity, we start with $g$ orthogonal to $\F$ under
$\norm{\cdot}_D$. At first, let us focus on the number of
\emph{unlabeled} points. We
could take $n=O(K \log \frac{d}{\delta} + \frac{K}{\eps})$ points from $D$ and
request the label of each point $x_i$. By
Theorem~\ref{thm:guarantee_AL_procedure} with the simpler
well-balanced sampling procedure mentioned above using $D' = D$, the
ERM $f'$ on these $n$ points is
$\epsilon \cdot \underset{D}{\E}[|g(x)|^2]$-close to $f$.

Then let us optimize the number of labeled samples. For $n$ random
points from $D$, let $D_0$ denote the uniform measurement on these
points. Although we cannot apply the linear-sample well-balanced
sampling procedure $P$ to the unknown $D$, we can apply it to
$D_0$. By Theorem~\ref{thm:guarantee_AL_procedure}, the ERM $\wt{f}$
of $P$ on $D_0$ satisfies
$\|\wt{f}-f'\|^2_{D_0} \le \epsilon \cdot
\underset{x \sim D_0}{\E}[|y(x)-f'(x)|^2]$. By the triangle inequality and the fact that $D_0$ is an good empirical estimation of $\FF$ under measurement $D$, this
gives
$\|f-\wt{f}\|^2_D \lesssim \epsilon \cdot \underset{D}{\E}[|g(x)|^2]$.

Notice that $f'$ only appears in the analysis and we do not need it in the calculation of $\wt{f}$ given $D_0$. By rescaling a constant factor of $\eps$, this leads to Theorem~\ref{thm:unknown_dist} proved in Section~\ref{sec:act_learning}.

\paragraph{Lower bounds.} We first prove a lower bound on the query
complexity using information theory. The Shannon-Hartley Theorem
indicates that under the \text{i.i.d.} Gaussian noise $N(0,1/\eps)$, for a
function $f$ with $|f(x)| \le 1$ at every point $x$, any observation
$y(x)=f(x)+N(0,1/\eps)$ obtains $O(\eps)$ information about
$f$. Because the dimension of $\FF$ is $d$, this indicates
$\Omega(d/\eps)$ queries is necessary to recover a function in $\FF$.

Next, for any $K$, $d$, and $\eps$ we construct a distribution $D$ and
dimension-$d$ linear family $\FF$ with condition number $K$ over $D$,
such that the sample complexity of achieving~\eqref{eq:approximation} is
$\Omega(K \log d + K/\eps)$.  The first term comes from the coupon
collector problem, and the second comes from the above query bound. We summarize the upper bounds and lower bounds for sample complexity and query complexity in Table~\ref{tab:lower_upper_bounds}.

\vspace{0.2in}

\begin{minipage}[t]{\linewidth}
  \centering
  \begin{tabular}{| l | l|l | l |}
    \hline
    & Optimal value & Lower bound & Upper bound \\
    \hline
    Query complexity & $\Theta(d/\eps)$ & Theorem~\ref{thm:information_lower_bound} & Theorem~\ref{thm:BSS} \\
    \hline
    Sample complexity & $\Theta(K\log d + K/\eps)$  &Theorem~\ref{lem:sample_complexity} & Theorem~\ref{thm:orthogonal_part} \\
    \hline
  \end{tabular}
  \captionof{table}{Lower bounds and upper bounds in different access models}\label{tab:lower_upper_bounds}
\end{minipage}

\paragraph{Signals with $k$-sparse Fourier transform.}
We now consider the nonlinear family $\F$ of functions with $k$-sparse
Fourier transform defined in \eqref{eq:def_k_FT}, over the distribution $D = [-1, 1]$.  As discussed
at~\eqref{def:kappa}, even for nonlinear function families, sampling
from $D_\F$ proportional to $\underset{f \in \FF}{\sup} \frac{|f(x)|^2}{\|f\|_D^2}$ improves the condition number from $K$ to $
  \kappa = \underset{x \in D}{\E} \underset{f \in \FF}{\sup} \frac{\abs{f(x)}^2}{\norm{f}_D^2},
$
which is $\tilde{O}(k)$ given Theorem~\ref{lem:bound_k_sparse_FT_x_middle} and $K=\tilde{O}(k^4)$.  

Before sketching the proof of Theorem~\ref{lem:bound_k_sparse_FT_x_middle}, let us revisit the $\Ot(k^4)$ bound for $K$ shown
in~\cite{CKPS}. The key step---Claim~5.2 in \cite{CKPS}---showed that
for any $\Delta>0$ and $f \in \FF$, $f(x)$ can be expressed as a
linear combination of $\big\{f(x+j \Delta)\mid j =1,\dotsc,l\}$ with
constant coefficients and $l=\tilde{O}(k^2)$. This upper bounds $|f(-1)|^2$ in terms of
$|f(-1 + \Delta)|^2 + \dotsb + |f(-1 + l \cdot \Delta)|^2$ and then 
$|f(-1)|^2/\|f\|_D^2$ by integrating $\Delta$ from $0$ to $2/l$.

The improvement of Theorem~\ref{lem:bound_k_sparse_FT_x_middle} contains two steps. In the
first step, we show that $f(x)$ can be expressed as a
constant-coefficient linear combination of the elements of an
$O(k)$-length arithmetic sequence on both sides of $x$, namely,
$\{f(x- 2k \cdot \Delta),\dotsc,f(x+2k \cdot \Delta)\} \setminus
f(x)$.  This is much shorter than the $\Ot(k^2)$ elements required by
\cite{CKPS} for the one-sided version, and provides an
$\tilde{O}(k^2)$ factor improvement. Next we find $k$ such linear
combinations that are almost orthogonal to each other to remove the
extra $k$ factor. These two let us show that
\[
 \sup_{f \in \F} \frac{\abs{f(x)}^2}{\norm{f}_D^2} = O\left(\frac{k \log k}{1 - \abs{x}} \right)
\]
for any $x \in (-1, 1)$.  This leads to $\kappa = O(k \log^2 k)$, which appears in Theorem~\ref{thm:kappa_sparse_FT} of Section~\ref{sec:k_sparse_FT}.






\section{Notation}\label{sec:preliminaries}
We use $[k]$ to denote the subset $\{1,2,\dotsc,k\}$ and $1_{E} \in \{0,1\}$ to denote the indicator function of an event $E$.

For a vector $\vec{v}=\big(v(1),\dotsc,v(m)\big) \in \mathbb{C}^m$, let $\|\vec{v}\|_k$ denote the $\ell_k$ norm, i.e., $\big(\sum_{i \in [m]} |v(i)|^k \big)^{1/k}$. 

Given a self-adjoint matrix $A \in \mathbb{C}^{m \times m}$, let $\|A\|$ denote the operator norm $\|A\|=\max_{\vec{v} \neq \vec{0}} \frac{\|A \vec{v}\|_{2}}{\|\vec{v}\|_{2}}$ and $\lambda(A)$ denote all eigenvalues of $A$. For convenience, let $\lambda_{\min}(A)$ and $\lambda_{\max}(A)$ denote the smallest eigenvalue and the largest eigenvalue of $A$. Given a matrix $B$, let $B^*$ denote the conjugate transpose of $B$, i.e., $B^*(j,i)=\overline{B(i,j)}$.

Given a function $f$ with domain $G$ and a distribution $D$ over  $G$, we use $\|f\|_D$ to denote the expected $\ell_2$ norm of $f(x)$ where $x \sim D$, i.e., $\|f\|_D=\big( \underset{x \sim D}{\E}[|f(x)|^2]\big)^{1/2}$. Given a sequence $S=(x_1,\dotsc,x_m)$ (allowing repetition in $S$) and corresponding weights $(w_1,\dotsc,w_m)$, let $\|f\|^2_{S,w}$ denote the weighted $\ell_2$ norm $\sum_{j=1}^m w_j \cdot |f(x_j)|^2$. For convenience, we omit $w$ if it is a uniform distribution on $S$, i.e., $\|f\|_{S}=\left( \E_{i \in [m]} \big[ |f(x_i)|^2 \big] \right)^{1/2}.$

\paragraph{Weights between different distributions.} Given a distribution $D$, to estimate $\|h\|_D^2$ of a function $h$ through random samples from $D'$, we use the following notation to denote the re-weighting of $h$ between $D'$ and $D$.
\begin{definition}
For any distribution $D'$ over the domain $G$ and any function $h : G \rightarrow \mathbb{C}$, let $h^{(D')}(x)=\sqrt{\frac{D(x)}{D'(x)}} \cdot h(x)$ such that $
\underset{x \sim D'}{\E} \left[ |h^{(D')}(x)|^2\right]=\underset{x \sim D'}{\E} \left[\frac{D(x)}{D'(x)} |h(x)|^2 \right]= \underset{x \sim D}{\E}\left[|h(x)|^2\right]. 
$
When the family $\FF$ and $D$ is clear, we use $K_{D'}$ to denote the condition number of sampling from $D'$, i.e., 
$$
K_{D'}=\underset{x}{\sup} \left\{ \underset{h \in \mathcal{F}}{\sup} \left\{ \frac{|h^{(D')}(x)|^2}{\|h^{D'}\|_{D'}^2} \right\}\right\}=\underset{x}{\sup} \bigg\{ \frac{D(x)}{D'(x)} \cdot \underset{h \in \mathcal{F}}{\sup} \big\{ \frac{|h(x)|^2}{\|h\|_2^2} \big\} \bigg\}.
$$
\end{definition}
By the same reason, for a random sample $x$ from distribution $D'$, we always use $w_x=\frac{D(x)}{D'(x)}$ to re-weight the sample $x$ such that it keeps the same expectation:
$$
\E_{x \sim D'}[w_x \cdot |h(x)|^2]=\E_{x \sim D'}[\frac{D(x)}{D'(x)} \cdot |h(x)|^2]=\E_{x \sim D}[|h(x)|^2]=\|h\|_D^2.
$$


\section{Recovery Guarantee for Well-Balanced Samples}\label{sec:agnostic_learning}


In this section, we show for well-balanced sampling procedures (per
Definition~\ref{def:procedure_agnostic_learning}) that the weighted ERM of a good execution (per Definition~\ref{def:weightederm}) approximates the true risk minimizer, and
hence the true signal. For generality, we first consider
points and labels from a joint distribution $(D,Y)$.

\restate{thm:guarantee_AL_procedure}

Next, we provide a corollary for specific kinds of noise. In the
first case, we consider noise functions representing independently
mean-zero noise at each position $x$ such as i.i.d.~Gaussian
noise. Second, we consider arbitrary noise functions on the domain.

\begin{corollary}\label{cor:specific_noise_AL_procedure}
  Given a linear family $\FF$ and distribution $D$, let
  $y(x) = f(x) + g(x)$ for $f \in \FF$ and $g$ a randomized function.
  Let $P$ be an $\eps$-well-balanced sampling procedure for $\FF$ and
  $D$.  
  The weighted ERM $\wt{f}$ of a good execution of $P$ satisfies 
\begin{enumerate}
\item $\|\tilde{f}-f\|_D^2 \le \eps \cdot \underset{g}{\E}[\|g\|_D^2]$ in expectation, when $g(x)$ is a random function from $G$ to $\mathbb{C}$ where each $g(x)$ is an independent random variable with $\underset{g}{\E}[g(x)]=0$.
\item With probability 0.99, $\|\tilde{f}-f\|_D \le (1+O(\eps)) \cdot \|g\|_D$ for any other noise function $g$.
\end{enumerate} 
\end{corollary}

In the rest of this section, we prove Theorem~\ref{thm:guarantee_AL_procedure} in Section~\ref{sec:proof_AL_procedure} and Corollary~\ref{cor:specific_noise_AL_procedure} in Section~\ref{sec:proof_cor_AL_procedure}. We discuss the speedup of the calculation of the ERM in Section~\ref{sec:running_time}.


\subsection{Proof of Theorem~\ref{thm:guarantee_AL_procedure}}\label{sec:proof_AL_procedure}
We introduce a few more notation in this proof. Given $\FF$ and the measurement $D$, let $\{v_1,\ldots,v_d\}$ be a fixed orthonormal basis of $\FF$, where inner products are taken under the distribution $D$, i.e., $\underset{x \sim D}{\E}[v_i(x) \cdot \overline{v_j(x)}]=1_{i=j}$ for any $i,j \in [d]$. For any function $h \in \FF$, let $\alpha(h)$ denote the coefficients $(\alpha(h)_1,\dotsc,\alpha(h)_d)$ under the basis $(v_1,\ldots,v_d)$ such that $h=\sum_{i=1}^d \alpha(h)_i \cdot v_i$ and $\|\alpha(h)\|_2=\|h\|_D$. 

We characterize the first property in Definition~\ref{def:procedure_agnostic_learning} of \emph{well-balanced sampling procedures} as bounding the eigenvalues of $A^* \cdot A$, where $A$ is the $m \times d$ matrix defined as $A(i,j)=\sqrt{w_i} \cdot v_j(x_i)$. 

\begin{lemma}\label{lem:operator_estimation}
For any $\eps>0$, given $S=(x_1,\dotsc,x_m)$ and their weights $(w_1,\dotsc,w_m)$, let $A$ be the $m \times d$ matrix defined as $A(i,j)=\sqrt{w_i} \cdot v_j(x_i)$.  Then
\[ 
\|h\|^2_{S,w}:=\sum_{j=1}^m w_j \cdot |f(x_j)|^2 \in [1 \pm \eps] \cdot \|h\|_D^2 \text{ \quad for every } h \in \FF
\] if and only if the eigenvalues of $A^* A$ are in $[1-\eps,1+\eps]$.
\end{lemma}

\begin{proof}
Notice that
\begin{equation}\label{eq:matrix_coefficient}
A \cdot \alpha(h)  = \big( \sqrt{w_1} \cdot h(x_1) ,\dotsc,\sqrt{w_m} \cdot h(x_m) \big).
\end{equation}
Because
$$
\|h\|^2_{S,w}=\sum_{i=1}^m w_i |h(x_i)|^2=\|A \cdot \alpha(h)\|_2^2=\alpha(h)^* \cdot (A^* \cdot A) \cdot \alpha(h) \in [\lambda_{\min}(A^* \cdot A),\lambda_{\max}(A^* \cdot A)] \cdot \|h\|_D^2
$$ 
and $h$ is over the linear family $\FF$, these two properties are equivalent.

\end{proof}

Next we consider the calculation of the weighted ERM $\wt{f}$. Given the weights $(w_1,\cdots,w_m)$ on $(x_1,\ldots,x_m)$ and labels $(y_1,\ldots,y_m)$, let $\vec{y}_w$ denote the vector of weighted labels $(\sqrt{w_1} \cdot y_1,\ldots,\sqrt{w_m} \cdot y_m)$. From \eqref{eq:matrix_coefficient}, the empirical distance $\|h-(y_1,\dotsc,y_m)\|^2_{S,w}$ equals $\|A \cdot \alpha(h) - \vec{y}_w \|^2_{2} \text{ for any } h \in \FF$. The function $\wt{f}$ minimizing $\|h-(y_1,\dotsc,y_m)\|_{S,w}= \|A \cdot \alpha(h) - \vec{y}_w \|_{2}$ overall all $h \in \FF$ is the pseudoinverse of $A$ on $\vec{y}_w$, i.e., 
$$
\alpha(\wt{f})=(A^* \cdot A)^{-1} \cdot A^* \cdot \vec{y}_w \text{ and } \wt{f}=\sum_{i=1}^d \alpha(\wt{f})_i \cdot v_i.
$$ 

Finally, we consider the distance between $f=\underset{h \in \FF}{\argmin} \big\{ \underset{(x,y) \sim (D,Y)}{\E} [|h(x)-y|^2]\big\}$ and $\wt{f}$. For convenience, let $\vec{f}_w=\big(\sqrt{w_1} \cdot f(x_1),\ldots, \sqrt{w_m} \cdot f(w_m) \big)$. Because $f \in \FF$, $(A^* \cdot A)^{-1} \cdot A^* \cdot \vec{f}_w=\alpha(f)$. This implies
\[
\|\wt{f}-f\|_D^2 = \|\alpha(\wt{f})-\alpha(f)\|_2^2 = \|(A^* \cdot A)^{-1} \cdot A^* \cdot (\vec{y}_w-\vec{f}_w) \|_2^2.
\]
We assume $\lambda\big( (A^* \cdot A)^{-1} \big)$ is bounded and consider $\|A^* \cdot (\vec{y}_w-\vec{f}_w) \|_2^2$.

\begin{lemma}\label{lem:guarantee_dist}
Let $P$ be an random sampling procedure terminating in $m$ iterations ($m$ is not necessarily fixed) that in every iteration $i$, it provides a coefficient $\alpha_i$ and a distribution $D_i$ to sample $x_i \sim D_i$. Let the weight $w_i=\alpha_i \cdot \frac{D(x_i)}{D_i(x_i)}$
and $A \in \mathbb{C}^{m \times d}$ denote the matrix $A(i,j)=\sqrt{w_i} \cdot v_j(x_i)$. Then for $f=\underset{h \in \FF}{\arg\min} \underset{(x,y)\sim (D,Y)}{\E}[|y-h(x)|^2]$,
\[
\E_P\left[ \|A^*(\vec{y}_w - \vec{f}_{S,w})\|_2^2\right] \le \sup_{P} \big\{\sum_{i=1}^m \alpha_i \big\} \cdot \max_j \big\{ \alpha_j \cdot K_{D_j} \big\} \underset{(x,y)\sim (D,Y)}{\E}[|y-f(x)|^2],
\]
where $K_{D_i}$ is the condition number for samples from $D_i$: $K_{D_i}=\underset{x}{\sup} \bigg\{ \frac{D(x)}{D_i(x)} \cdot \underset{v \in \mathcal{F}}{\sup} \big\{ \frac{|v(x)|^2}{\|v\|_2^2} \big\} \bigg\}$.
\end{lemma}

\begin{proof}
For convenience, let $g_j$ denote $y_j-f(x_j)$ and $\vec{g}_w \in \mathbb{C}^m$ denote the vector $\bigg(\sqrt{w_j} \cdot g_j|_{j=1,\ldots,m}\bigg)=\vec{y}_w - \vec{f}_{S,w}$ for $j \in [m]$ such that $A^* \cdot (\vec{y}_w - \vec{f}_{S,w})=A^* \cdot \vec{g}_w$. 
\begin{align*}
\E[\|A^* \cdot \vec{g}_w\|_2^2] & = \E\left[\sum_{i=1}^d \big(\sum_{j=1}^m A^*(i,j) \vec{g}_w(j) \big)^2 \right]\\
& = \sum_{i=1}^d \E \left[ \big( \sum_{j=1}^m w_j \overline{v_i(x_j)} \cdot g_j \big)^2\right] = \sum_{i=1}^d \E \left[ \sum_{j=1}^m w_j^2 \cdot |v_i(x_j)|^2 \cdot |g_j|^2 \right],
\end{align*}
where the last step uses the following fact 
$$
\underset{x_j \sim D_j}{\E}[w_j \overline{v_i(x_j)} \cdot g_j]=\underset{x_j \sim D_j}{\E}\big[\alpha_j \cdot \frac{D(x_j)}{D_j(x_j)} \overline{v_i(x_j)} g_j\big]=\alpha_j \cdot \underset{x_j \sim D,y_j \sim Y(x_j)}{\E}\big[\overline{v_i(x_j)} (y_j-f(x_j))\big]=0.
$$ 

We swap $i$ and $j$: 
\begin{align*}
\sum_{i=1}^d \E \left[ \sum_{j=1}^m w_j^2 \cdot |v_i(x_j)|^2 \cdot |g_j|^2\right] & = \sum_{j=1}^m \E \left[ \sum_{i=1}^d  w_j |v_i(x_j)|^2 \cdot w_j |g_j|^2 \right]\\
& \le \sum_{j=1}^m \sup_{x_j} \left\{ w_j \sum_{i=1}^d |v_i(x_j)|^2 \right\} \cdot \E \left[ w_j \cdot |g_j|^2 \right].
\end{align*}

For $\E \left[ w_j \cdot |g_j|^2 \right]$, it equals $\underset{x_j \sim D_j,y_j \sim Y(x_j)}{\E} \left[ \alpha_j \cdot \frac{D(x_j)}{D_j(x_j)} \big|y_j-f(x_j)\big|^2 \right]=\alpha_j \cdot \underset{x_j \sim D,y_j \sim Y(x_j)}{\E}\big[ \big|y_j-f(x_j) \big|^2\big]$.

For $\sup_{x_j} \left\{ w_j \sum_{i=1}^d |v_i(x_j)|^2 \right\}$, we bound it as
\[
\sup_{x_j} \left\{ w_j \sum_{i=1}^d |v_i(x_j)|^2 \right\} = \sup_{x_j} \left\{ \alpha_j \cdot \frac{D(x_j)}{D_j(x_j)} \sum_{i=1}^d |v_i(x_j)|^2 \right\} = \alpha_j \sup_{x_j} \left\{   \frac{D(x_j)}{D_j(x_j)} \cdot \sup_{h \in \mathcal{F}} \big\{ \frac{|h(x_j)|^2}{\|h\|_D^2} \big\} \right\} = \alpha_j \cdot K_{D_j}.
\]
We use the fact $\underset{h \in \mathcal{F}}{\sup} \big\{ \frac{|h(x_j)|^2}{\|h\|_D^2} \big\}=\underset{(a_1,\ldots,a_d)}{\sup} \big\{ \frac{|\sum_{i=1}^d a_i v_i(x_j)|^2}{\sum_{i=1}^d |a_i|^2}\big\} = \frac{(\sum_{i=1}^d |a_i|^2) (\sum_{i=1}^d |v_i(x_j)|^2)}{\sum_{i=1}^d |a_i|^2}$ by the Cauchy-Schwartz inequality. From all discussion above, we have
\[
\E[\|A^* \cdot \vec{g}_w\|_2^2] \le \sum_j \left( \alpha_j K_{D_j} \cdot  \alpha_j \cdot  \underset{(x,y)\sim (D,Y)}{\E}[|y-f(x)|^2] \right) \le (\sum_j \alpha_j) \max_j \big\{ \alpha_j K_{D_j} \big\}  \cdot \underset{(x,y)\sim (D,Y)}{\E}[|y-f(x)|^2].
\]
\end{proof}
We combine all discussion above to prove Theorem~\ref{thm:guarantee_AL_procedure}.

\begin{proofof}{Theorem~\ref{thm:guarantee_AL_procedure}}
We assume the first property $\lambda(A^* \cdot A) \in [1-1/4,1+1/4]$ from Definition~\ref{def:weightederm}. On the other hand, $\E[\|A^* \cdot (\vec{y}_w-\vec{f}_w)\|_2^2] \le \epsilon/2 \cdot \underset{(x,y)\sim (D,Y)}{\E}[|y-f(x)|^2]$ from Lemma~\ref{lem:guarantee_dist}. Conditioned on the first property, we know it is still at most $\frac{\eps}{2 \cdot 0.9} \cdot \underset{(x,y)\sim (D,Y)}{\E}[|y-f(x)|^2]$. This implies $\E\big[ \|(A^* \cdot A)^{-1} \cdot A^* \cdot (\vec{y}_w-\vec{f}_w) \|_2^2 \big] \le \epsilon \cdot \underset{(x,y)\sim (D,Y)}{\E}[|y-f(x)|^2]$.
\end{proofof}

\subsection{Proof of Corollary~\ref{cor:specific_noise_AL_procedure}}\label{sec:proof_cor_AL_procedure}
For the first part, let $(D,Y)=\big( D , f(x)+g(x) \big)$ be our joint distribution of $(x,y)$. Because the expectation $\E[g(x)]=0$ for every $x \in G$, $\underset{v \in V}{\arg\min} \underset{(x,y)\sim (D,Y)}{\E}[|y-v(x)|^2]=f$. From Theorem~\ref{thm:guarantee_AL_procedure}, for $\alpha(\wt{f})=(A^* \cdot A)^{-1} \cdot A^* \cdot \vec{y}_w$ and $m=O(d/\eps)$,
\[
\|\wt{f}-f\|^2_D = \|\alpha(\wt{f})-\alpha(f)\|^2_2 \le \eps \cdot \E_{(x,y) \sim (D,Y)}[|y-f(x)|^2]=\eps \cdot \E[\|g\|_D^2].
\]

For the second part, let $g^{\parallel}$ be the projection of $g(x)$ to $\FF$ and $g^{\bot}=g - g^{\parallel}$ be the orthogonal part to $\FF$. Let $\alpha(g^{\parallel})$ denote the coefficients of $g^{\parallel}$ in the fixed orthonormal basis $(v_1,\dotsc,v_d)$ so that $\|\alpha(g^{\parallel})\|_2=\|g^{\parallel}\|_D$. We decompose $\vec{y}_w=\vec{f}_w+\vec{g}_w=\vec{f}_w + \vec{g^{\parallel}}_w + \vec{g^{\bot}}_w$. Therefore
$$
\alpha(\wt{f})=(A^* A)^{-1} \cdot A^* \cdot (\vec{f}_w + \vec{g^{\parallel}}_w + \vec{g^{\bot}}_w)= \alpha(f) + \alpha(g^{\parallel}) + (A A^*)^{-1} A^* \cdot \vec{g^{\bot}}_w.
$$ The distance $\|\wt{f}-f\|_D=\|\alpha(\wt{f})-\alpha(f)\|_2$ equals 
\[
\|(A^* A)^{-1} \cdot A^* \cdot \vec{y}_w-\alpha(f)\|_2=\| \alpha(f) + \alpha(g^{\parallel}) + (A^* A)^{-1} \cdot A^* \cdot \vec{g^{\bot}}_w - \alpha(f) \|_2=\|\alpha(g^{\parallel}) + (A^* A)^{-1} \cdot A^* \cdot \vec{g^{\bot}}_w\|_2.
\]

We apply Markov's inequality to Theorem~\ref{thm:guarantee_AL_procedure}: with probability $0.99$, $\|(A^* A)^{-1} \cdot A^* \cdot \vec{g^{\bot}}_w\|_2 \le 10 \sqrt{\eps} \cdot \|g^{\bot}\|_D$. Thus 
\begin{align*}
\|(A^* A)^{-1} \cdot A^* \cdot \vec{y}_w-\alpha(f)\|_2& =\|\alpha(g^{\parallel}) + (A^* A)^{-1} \cdot A^* \cdot \vec{g^{\bot}}_w\|_2\\
& \le \|g^{\parallel}\|_D + 10 \sqrt{\eps} \cdot \|g^{\bot}\|_D.
\end{align*}
Let $1-\beta$ denote $\|g^{\parallel}\|_D/\|g\|_D$ such that $\|g^{\bot}\|_D/\|g\|_D=\sqrt{2\beta - \beta^2}$. We rewrite it as
$$
\left(1 - \beta+ 10 \sqrt{\eps} \cdot \sqrt{2\beta - \beta^2}\right)\|g\|_D \le (1-\beta + 10 \sqrt{\eps} \cdot \sqrt{2\beta}) \|g\|_D \le \left(1-(\sqrt{\beta} - 5\sqrt{2\eps})^2 + 50 \eps \right) \|g\|_D.
$$
From all discussion above, $\|\wt{f}-f\|_D=\|\alpha(\wt{f})-\alpha(f)\|_2=\|(A^* A)^{-1} \cdot A^* \cdot \vec{y}_w-\alpha(f)\|_2 \le (1+O(\eps)) \|g\|_D$.


\subsection{Running time of finding ERM}\label{sec:running_time}
Given the orthonormal basis $v_1,\cdots,v_d$ of $\FF$ under $D$, the ERM on noisy observations $y(x_1),\cdots,y(x_m)$ with weights $w_1,\cdots,w_m$ is $(A^* A)^{-1} \cdot A^* \cdot \vec{y}_w$ for $A \in \mathbb{C}^{m \times d}$ defined as $A(i,j)=\sqrt{w_i} \cdot v_j(x_i)$ and $\vec{y}_w=\big( \sqrt{w_1} \cdot y(x_1),\ldots,\sqrt{w_m} \cdot y(x_m) \big)$. Since well-balanced procedures guarantee $\lambda(A^* \cdot A) \in [3/4,5/4]$, we could calculate an $\delta$-approximation of the ERM using Taylor expansion $(A^* A)^{-1} \approx\sum_{i=0}^{t} (I - A^* A)^i$ for $t=O(\log \frac{1}{\delta})$. This saves the cost of calculating the inverse $(A^* A)^{-1}$ and  improves it to $O(m \cdot d \cdot \log \frac{1}{\delta})$ for any linear family. 
\begin{observation}\label{lem:taylor_expan}
Let $A$ be a $m \times d$ matrix defined as $A(i,j)=\sqrt{w_i} \cdot v_j(x_i)$ with $\lambda(A^* \cdot A) \in [3/4,5/4]$. Given $\delta$, for $t=O(\log 1/\delta)$ and any vector $\vec{y} \in \mathbb{R}^m$, 
$$
\| (A^* \cdot A)^{-1} \cdot A^* \cdot \vec{y} - \big(\sum_{i=0}^t (I - A^* \cdot A)^i\big) \cdot A^* \cdot \vec{y}\|_2 \le \delta \cdot \| (A^* \cdot A)^{-1} \cdot A^* \cdot \vec{y}\|_2.
$$
\end{observation}

\section[]{A Linear-Sample Algorithm for Known $D$}\label{sec:BSS}
We provide a well-balanced sampling procedure with a linear number of
random samples in this section.  The procedure requires knowing the
underlying distribution $D$, which makes it directly useful in the
query setting or the ``fixed design'' active learning setting, where
$D$ can be set to the empirical distribution $D_0$.

\begin{lemma}\label{lemma:BSS}
  Given any dimension $d$ linear space $\FF$, any distribution $D$ over the domain of $\FF$, and any $\eps>0$,  there exists an efficient $\eps$-\emph{well-balanced sampling procedure} that terminates in $O(d/\eps)$ rounds with probability $1-\frac{1}{200}$.
\end{lemma}

Theorem~\ref{thm:BSS} follows from Theorem~\ref{thm:guarantee_AL_procedure} using the above \emph{well-balanced sampling procedure}. We state the following version for specific types of noise after plugging the above \emph{well-balanced sampling procedure} in Corollary~\ref{cor:specific_noise_AL_procedure}.

\begin{theorem}
  Given any dimension $d$ linear space $\FF$ of functions and any
  distribution $D$ on the domain of $\FF$, let $y(x)=f(x)+g(x)$ be our
  observed function, where $f \in \FF$ and $g$ denotes a noise
  function. For any $\eps>0$, there exists an efficient algorithm that
  observes $y(x)$ at $m=O(\frac{d}{\eps})$ points and outputs $\wt{f}$
  such that in expectation,
\begin{enumerate}
\item $\|\tilde{f}-f\|_D^2 \le \eps \cdot \underset{g}{\E}[\|g\|_D^2]$, when $g(x)$ is a random function from $G$ to $\mathbb{C}$ where each $g(x)$ is an independent random variable with $\underset{g}{\E}[g(x)]=0$.
\item $\|\tilde{f}-f\|_D \le (1+\eps) \cdot \|g\|_D$ for any other noise function $g$.
\end{enumerate} 
\end{theorem}

We show how to extract the coefficients $\alpha_1,\cdots,\alpha_m$ from the randomized BSS algorithm by \cite{LeeSun} in Algorithm~\ref{alg:BSS}. Given $\epsilon$, the linear family $\FF$, and the distribution $D$, we fix $\gamma=\sqrt{\epsilon}/C_0$ for a constant $C_0$ and $v_1,\ldots,v_d$ to be an orthonormal basis of $\FF$ in this section. For convenience, we use $v(x)$ to denote the vector $\big(v_1(x),\ldots,v_d(x) \big)$.

In the rest of this section, we prove Lemma~\ref{lemma:BSS} in Section~\ref{sec:proof_BSS}.

\begin{algorithm}
\caption{A well-balanced sampling procedure based on Randomized BSS}\label{alg:BSS}
\begin{algorithmic}[1]
\Procedure{\textsc{RandomizedSamplingBSS}}{$\FF,D,\epsilon$}
\State Find an orthonormal basis $v_1,\ldots,v_d$ of $\FF$ under $D$;
\State Set $\gamma=\sqrt{\epsilon}/C_0$ and $\midd=\frac{4d/\gamma}{1/(1-\gamma)-1/(1+\gamma)}$;
\State $j = 0; B_0=0$;
\State $l_0=-2d/\gamma; u_0=2d/\gamma$;
\While {$u_{j+1}-l_{j+1}<8 d/\gamma$};
\State $\Phi_j = \Tr(u_j I - B_j)^{-1} + \Tr(B_j - l_j I)^{-1}$; \hfill $\triangleright$ The potential function at iteration $j$.
\State Set the coefficient $\alpha_j=\frac{\gamma}{\Phi_j} \cdot \frac{1}{\midd}$;
\State Set the distribution $D_j(x)=D(x) \cdot \bigg(v(x)^\top (u_j I - B_j)^{-1} v(x) + v(x)^\top (B_j - l_j I)^{-1} v(x) \bigg)/\Phi_j$ for $v(x)=\big(v_1(x),\ldots,v_d(x) \big)$;
\State Sample $x_j \sim D_j$ and set a scale $s_j=\frac{\gamma}{\Phi_j} \cdot \frac{D(x)}{D_j(x)}$;
\State $B_{j+1}=B_j + s_j \cdot v(x_j) v(x_j)^\top$;
\State $u_{j+1}=u_j + \frac{\gamma}{\Phi_j (1-\gamma)}; \quad l_{j+1}=l_j + \frac{\gamma}{\Phi_j ( 1 +\gamma)}$;
\State $j=j+1$;
\EndWhile
\State $m=j$;
\State Assign the weight $w_j=s_j/\midd$ for each $x_j$;
\EndProcedure
\end{algorithmic}
\end{algorithm}

\subsection{Proof of Lemma~\ref{lemma:BSS}}\label{sec:proof_BSS}
We state a few properties of randomized BSS \cite{batson2012twice,LeeSun} that will be used in this proof. The first property is that matrices $B_1,\ldots,B_m$ in Procedure \textsc{RandomizedBSS} always have bounded eigenvalues.

\begin{lemma}[\cite{batson2012twice,LeeSun}]\label{lem:eigenvalues}
For any $j \in [m]$, $\lambda(B_j) \in (l_j,u_j)$.
\end{lemma}

Lemma 3.6 and 3.7 of \cite{LeeSun} shows that with high probability, the while loop in Procedure \textsc{RandomizedSamplingBSS} finishes within $O(\frac{d}{\gamma^2})$ iterations and guarantees the last matrix $B_m$ is well-conditioned, i.e., $\frac{\lambda_{\max}(B_m)}{\lambda_{\min}(B_m)} \le \frac{u_m}{l_m} \le 1+O(\gamma)$.

\begin{lemma}[\cite{LeeSun}]\label{lem:well_condition_BSS}
There exists a constant $C$ such that with probability at least $1-\frac{1}{200}$, Procedure \textsc{RandomizedSamplingBSS} takes at most $m=C \cdot d/\gamma^2$ random points $x_1,\ldots,x_m$ and guarantees that $\frac{u_m}{l_m} \le 1+8\gamma$.
\end{lemma}
We first show that $(A^* \cdot A)$ is well-conditioned from the definition of $A$. We prove that our choice of $\midd$ is very close to $\sum_{j=1}^m \frac{\gamma}{\phi_j}=\frac{u_m+l_m}{\frac{1}{1-\gamma} + \frac{1}{1+\gamma}} \approx \frac{u_m+l_m}{2}$.

\begin{claim}\label{clm:estimate_mid}
After exiting the while loop in Procedure \textsc{RandomizedBSS}, we always have
\begin{enumerate}
\item $u_m-l_m \le 9d/\gamma$. 
\item $(1-\frac{0.5 \gamma^2}{d}) \cdot \sum_{j=1}^m \frac{\gamma}{\phi_j} \le \midd \le \sum_{j=1}^m \frac{\gamma}{\phi_j}$.
\end{enumerate}
\end{claim}
\begin{proof}
Let us first bound the last term $\frac{\gamma}{\phi_m}$ in the while loop. Since $u_{m-1}-l_{m-1}<8d/\gamma$, $\phi_m \ge 2d \cdot \frac{1}{4d/\gamma} \ge \frac{\gamma}{2}$, which indicates the last term $\frac{\gamma}{\phi_m} \le 2$. Thus 
$$
u_m-l_m \le 8d/\gamma + 2(\frac{1}{1-\gamma} - \frac{1}{1+\gamma}) \le 8d/\gamma + 5 \gamma.
$$

From our choice $\midd=\frac{4d/\gamma}{1/(1-\gamma)-1/(1+\gamma)}=2d(1-\gamma^2)/\gamma^2$ and the condition of the while loop $u_m-l_m=\sum_{j=1}^m (\gamma/\phi_j) \cdot (\frac{1}{1-\gamma} - \frac{1}{1+\gamma})+4d/\gamma \ge 8d/\gamma$, we know 
$$
\sum_{j=1}^m \frac{\gamma}{\phi_j} \ge \midd = 2d(1-\gamma^2)/\gamma^2.
$$ 

On the other hand, since $u_{m-1}-l_{m-1}<8d/\gamma$ is in the while loop, $\sum_{j=1}^{m-1} \frac{\gamma}{\phi_j} < \midd$.  Hence 
$$
\midd > \sum_{j=1}^{m-1} \frac{\gamma}{\phi_j} \ge \sum_{j=1}^m \frac{\gamma}{\phi_j} - 2 \ge (1-0.5 \gamma^2/d) \cdot (\sum_{j=1}^m \frac{\gamma}{\phi_j}).
$$
\end{proof}

\begin{lemma}\label{lemma:eigenvalue_BSS}
Given $\frac{u_m}{l_m} \le 1+8\gamma$, $\lambda(A^* \cdot A) \in (1-5\gamma, 1+5 \gamma)$.
\end{lemma}
\begin{proof}
For $B_m=\sum_{j=1}^m s_j v(x_j) v(x_j)^{\top}$, $\lambda(B_m) \in (l_m,u_m)$ from Lemma~\ref{lem:eigenvalues}. At the same time, given $w_j=s_j/\midd$, 
$$
(A^* A)=\sum_{j=1}^m w_j v(x_j) v(x_j)^{\top}=\frac{1}{\midd} \cdot \sum_{j=1}^m s_j v(x_j) v(x_j)^{\top}=\frac{B_m}{\midd}.
$$
Since $\midd \in [1-\frac{3\gamma^2}{d},1] \cdot (\sum_{j=1}^m \frac{\gamma}{\phi_j}) = [1-\frac{3\gamma^2}{d},1] \cdot (\frac{u_m+l_m}{\frac{1}{1-\gamma}+\frac{1}{1+\gamma}}) \subseteq [1-2\gamma^2,1-\gamma^2] \cdot (\frac{u_m+l_m}{2})$ from Claim~\ref{clm:estimate_mid}, $\lambda(A^* \cdot A)=\lambda(B_m)/\midd \in (l_m/\midd, u_m/\midd) \subset (1-5\gamma,1+5\gamma)$ given $\frac{u_m}{l_m} \le 1+8\gamma$ in Lemma~\ref{lem:well_condition_BSS}.
\end{proof}

We finish the proof of Lemma~\ref{lemma:BSS} by combining all discussion above.

\begin{proofof}{Lemma~\ref{lemma:BSS}}
From Lemma~\ref{lem:well_condition_BSS} and Lemma~\ref{lemma:eigenvalue_BSS}, $m=O(d/\gamma^2)$ and $\lambda(A^* A) \in [1-1/4,1+1/4]$ with probability 0.995.

For $\alpha_i=\frac{\gamma}{\Phi_i} \cdot \frac{1}{\midd}$, we bound $\sum_{i=1}^m \frac{\gamma}{\Phi_i} \cdot \frac{1}{\midd}$ by 1.25 from the second property of Claim~\ref{clm:estimate_mid}. 

Then we bound $\alpha_j \cdot K_{D_j}$. We notice that $\underset{h \in \FF}{\sup} \frac{|h(x)|^2}{\|h\|_D^2}=\sum_{i \in [d]} |v_i(x)|^2$ for every $x \in G$ because $
\underset{h \in \FF}{\sup} \frac{|h(x)|^2}{\|h\|_D^2}=\underset{\alpha(h)}{\sup} \frac{\big| \sum_i \alpha(h)_i \cdot v_i(x) \big|^2}{\|\alpha(h)\|_2^2} = \sum_i |v_i(x)|^2$ by the Cauchy-Schwartz inequality.
This simplifies $K_{D_j}$ to $\sup_x \{ \frac{D(x)}{D_j(x)} \cdot \sum_{i=1}^d |u_i(x)|^2 \}$ and bounds $\alpha_j \cdot K_{D_j}$ by
\begin{align*}
& \frac{\gamma}{\Phi_j \cdot \midd} \cdot \sup_{x} \left\{ \frac{D(x)}{D_j(x)} \cdot \sum_{i=1}^d |v_i(x)|^2 \right\}\\
=& \frac{\gamma}{\midd} \cdot \sup_x \left\{ \frac{\sum_{i=1}^d |v_i(x)|^2}{v(x_j)^\top (u_j I - B_j)^{-1} v(x_j) + v(x_j)^\top (B_j - l_j I)^{-1} v(x_j)} \right\} \\
\le & \frac{\gamma}{\midd} \cdot \sup_x \left\{ \frac{\sum_{i=1}^d |v_i(x)|^2}{\lambda_{\min}\big((u_j I - B_j)^{-1}\big) \cdot \|v(x_j)\|_2^2 + \lambda_{\min}\big((B_j - l_j I)^{-1}\big) \cdot \|v(x_j)\|_2^2} \right\} \\
\le & \frac{\gamma}{\midd} \cdot \frac{1}{1/(u_j-l_j) + 1/(u_j-l_j)}\\
= & \frac{\gamma}{\midd} \cdot \frac{u_j-l_j}{2} \qquad \qquad (\text{apply the first property of Claim~\ref{clm:estimate_mid}})\\
\le & \frac{4.5 \cdot d}{\midd} \le 3 \gamma^2 = 3 \epsilon/C_0^2.
\end{align*}
By choosing $C_0=3$, this satisfies the second property of \emph{well-balanced sampling procedures}. At the same time, by Lemma~\ref{lem:operator_estimation}, Algorithm~\ref{alg:BSS} also satisfies the first property of \emph{well-balanced sampling procedures}.
\end{proofof}

\section{Performance of \text{i.i.d.} Distributions}\label{sec:single_batch}
Given the linear family $\FF$ of dimension $d$ and the measure of distance $D$, we provide a distribution $D_{\FF}$ with a condition number $K_{D_{\FF}}=d$.
\begin{lemma}\label{lem:d_ff}
Given any linear family $\FF$ of dimension $d$ and any distribution $D$, there always exists an explicit distribution $D_{\FF}$ such that the condition number
\[
K_{D_{\FF}}=\underset{x}{\sup} \bigg\{ \underset{h \in \mathcal{F}}{\sup} \big\{ \frac{D(x)}{D_{\FF}(x)} \cdot \frac{|h(x)|^2}{\|h\|_D^2} \big\} \bigg\}=d.
\]
\end{lemma}
Next, for generality, we bound the number of i.i.d.~random samples from an arbitrary distribution $D'$ to fulfill the requirements of \emph{well-balanced sampling procedures} in Definition~\ref{def:procedure_agnostic_learning}. 

\begin{lemma}\label{lem:agnostic_learning_single_distribution}
There exists a universal constant $C_1$ such that given any distribution $D'$ with the same support of $D$, any $\epsilon>0$, and any $\delta$, the random sampling procedure with $m=C_1(K_{D'} \log \frac{d}{\delta} + \frac{K_{D'}}{\eps})$ i.i.d.~random samples from $D'$ and coefficients $\alpha_1=\cdots=\alpha_m=1/m$ is an $\eps$-\emph{well-balanced sampling procedure} with probability $1-\delta$.
\end{lemma}

By Theorem~\ref{thm:guarantee_AL_procedure}, we state the following result, which will be used in active learning. For $G=\supp(D)$ and any $x \in G$, let
$Y(x)$ denote the conditional distribution $(Y|D=x)$ and
$(D',Y(D'))$ denote the distribution that first generates $x \sim D'$
then generates $y \sim Y(x)$.

\begin{theorem}\label{thm:orthogonal_part}
  Consider any dimension $d$ linear space $\FF$ of functions from a
  domain $G$ to $\mathbb{C}$. Let $(D,Y)$ be a joint distribution over
  $G \times \mathbb{C}$, and
  $f=\underset{h \in \FF}{\arg\min} \underset{(x,y)\sim
    (D,Y)}{\E}[|y-h(x)|^2]$.

  Let $D'$ be any distribution on $G$ and 
  $K_{D'}=\underset{x}{\sup} \bigg\{ \underset{h \in \mathcal{F}}{\sup} \big\{ \frac{D(x)}{D'(x)} \cdot \frac{|h(x)|^2}{\|h\|_D^2} \big\} \bigg\}$.  The weighted ERM $\wt{f}$ of 
  $m=O(K_{D'} \log d + \frac{K_{D'}}{\eps})$ random queries of
  $(D',Y(D'))$ with weights $w_i=\frac{D(x_i)}{m \cdot D'(x_i)}$ for each $i \in [m]$ satisfies
\[
\|\tilde{f}-f\|_D^2 = \underset{x \sim D}{\E}\big[|\tilde{f}(x)-f(x)|^2\big] \le \eps \cdot \underset{(x,y)\sim (D,Y)}{\E}\big[|y-f(x)|^2\big] \text{ with probability } 1-10^{-4}.
\]
\end{theorem}

We show the proof of Lemma~\ref{lem:d_ff} in Section~\ref{sec:construction_D_V} and the proof of Lemma~\ref{lem:agnostic_learning_single_distribution} in Section~\ref{sec:single_batch_proof}.


\subsection{Optimal Condition Number}\label{sec:construction_D_V}
We describe the distribution $D_{\FF}$ with $K_{D_{\FF}}=d$. We first observe that for any family $\FF$ (not necessarily linear), we could always scale down the condition number to $\kappa=\underset{x \sim D}{\E} \left[ \underset{h \in \FF: h \neq 0}{\sup} \frac{|h(x)|^2}{\|h\|^2_D} \right]$.

\begin{claim}\label{clm:new_conditioning}
For any family $\FF$ and any distribution $D$ on its domain, let $D_{\FF}$ be the distribution defined as $D_{\FF}(x)=\frac{D(x) \cdot \underset{h \in \FF: h \neq 0}{\sup} \frac{|h(x)|^2}{\|h\|^2_D}}{\kappa}$ with $\kappa$. The condition number $K_{D_{\FF}}$ is at most $\kappa$.
\end{claim}
\begin{proof}
For any $g \in \FF$ and $x$ in the domain $G$,
$$
\frac{|g(x)|^2}{\|g\|_D^2} \cdot \frac{D(x)}{D_{\FF}(x)}=\frac{\frac{|g(x)|^2}{\|g\|_D^2} \cdot D(x)}{\sup_{h \in \mathcal{F}} \frac{|h(x)|^2}{\|h\|_D^2} \cdot D(x) / \kappa} \le \kappa. 
$$
\end{proof}

Next we use the linearity of $\FF$ to prove $\kappa=d$. Let $\{v_1,\dotsc,v_d\}$ be any orthonormal basis of $\FF$, where inner products are taken under the distribution $D$. 
\begin{lemma}\label{lem:bound_D_V}
For any linear family $\FF$ of dimension $d$ and any distribution $D$, \[
\underset{x \sim D}{\E} \sup_{h \in \FF:\|h\|_D=1} |h(x)|^2 = d
\] such that
$D_{\FF}(x)=D(x) \cdot \underset{h \in \FF:\|h\|_D=1}{\sup} |h(x)|^2 /d$ has a condition number $K_{D_{\FF}}=d$.
Moreover, there exists an efficient algorithm to sample $x$ from $D_{\FF}$ and compute its weight $\frac{D(x)}{D_{\FF}(x)}$.
\end{lemma}
\begin{proof}
Given an orthonormal basis $v_1,\dotsc,v_d$ of $\FF$, for any $h \in \FF$ with $\|h\|_D=1$, there exists $c_1,\dotsc,c_d$ such that $h(x)=c_i \cdot v_i(x)$. Then for any $x$ in the domain, from the Cauchy-Schwartz inequality, 
\[
\sup_{h} \frac{|h(x)|^2}{\|h\|_D^2}= \sup_{c_1,\dotsc,c_d} \frac{|\sum_{i \in [d]} c_i v_i(x)|^2}{\sum_{i \in [d]} |c_i|^2} = \frac{(\sum_{i \in [d]} |c_i|^2) \cdot (\sum_{i \in [d]} |v_i(x)|^2)}{\sum_{i \in [d]} |c_i|^2} = \sum_{i \in [d]} |v_i(x)|^2.
\]
This is tight because there always exist $c_1=\overline{v_1(x)}, c_2=\overline{v_2(x)},\dotsc, c_d=\overline{v_d(x)}$ such that $|\underset{i \in [d]}{\sum} c_i v_i(x)|^2=(\underset{i \in [d]}{\sum} |c_i|^2) \cdot (\underset{i \in [d]}{\sum} |v_i(x)|^2)$.
Hence 
$$
\underset{x \sim D}{\E} \underset{h \in \FF:h\neq 0}{\sup} \frac{|h(x)|^2}{\|h\|_D^2}=\underset{x \sim D}{\E}\big[\sum_{i \in [d]} |v_i(x)|^2\big]=d.
$$
By Claim~\ref{clm:new_conditioning}, this indicates $K_{D_{\FF}}=d$. At the same time, this calculation indicates
\[
D_{\FF}(x)=\frac{D(x) \cdot \underset{\|h\|_D=1}{\sup} |h(x)|^2}{d}=\frac{D(x) \cdot \sum_{i \in [d]} |v_i(x)|^2}{d}.
\] 
We present our sampling procedure in Algorithm~\ref{alg:sample}.


\begin{algorithm}[t]
\caption{SampleDF}\label{alg:sample}
\begin{algorithmic}[1]
\Procedure{\textsc{GeneratingDF}}{$\FF=\text{span}\{v_1,\dotsc,v_d\},D$}
\State Sample $j \in [d]$ uniformly.
\State Sample $x$ from the distribution $W_j(x)=D(x) \cdot |v_j(x)|^2$.
\State Set the weight of $x$ to be $\frac{d}{\sum_{i=1}^d |v_i(x)|^2}$.
\EndProcedure
\end{algorithmic}
\end{algorithm}

\end{proof}


\subsection{Proof of Lemma~\ref{lem:agnostic_learning_single_distribution}}\label{sec:single_batch_proof}
We use the matrix Chernoff theorem to prove the first property in Definition~\ref{def:procedure_agnostic_learning}. We still use $A$ to denote the $m \times d$ matrix $A(i,j)=\sqrt{w_i} \cdot v_j(x_i)$.
\begin{lemma}\label{lmm:general_distribution}
Let $D'$ be an arbitrary distribution over $G$ and 
\begin{equation}\label{eq:def_K}
K_{D'}=\sup_{h \in \FF: h \neq 0} \sup_{x \in G} \frac{|h^{(D')}(x)|^2}{\|h\|_D^2}.
\end{equation}
There exists an absolute constant $C$ such that for any $n \in \mathbb{N}^+$, linear family $\FF$ of dimension $d$, $\eps\in (0,1)$ and $\delta \in (0,1)$, when $S=(x_1,\dotsc,x_{m})$ are independently from the distribution $D'$ with $m \ge \frac{C}{\eps^2} \cdot K_{D'} \log \frac{d}{\delta}$ and $w_j=\frac{D(x_j)}{m \cdot D'(x_j)}$ for each $j \in [m]$, the $m \times d$ matrix $A(i,j)=\sqrt{w_i} \cdot v_j(x_i)$ satisfies
\[
\|A^* A- I\| \le \eps \text{ with probability at least } 1-\delta.
\]
\end{lemma}

Before we prove Lemma~\ref{lmm:general_distribution}, we state the following version of the matrix Chernoff bound.
\begin{theorem}[Theorem 1.1 of \cite{Tropp}]\label{thm:matrix_chernoff}
Consider a finite sequence $\{X_k\}$ of independent, random, self-adjoint matrices of dimension $d$. Assume that each random matrix satisfies
$$ X_k \succeq 0 \quad \text{ and } \quad \lambda(X_k) \le R.$$
Define $\mu_{\min}=\lambda_{\min}(\sum_k \E[X_k])$ and $\mu_{\max}=\lambda_{\max}(\sum_k \E[X_k])$. Then
\begin{align}
\Pr\left\{\lambda_{\min}(\sum_k X_k) \le (1-\delta)\mu_{\min}\right\} & \le d \left( \frac{e^{-\delta}}{(1-\delta)^{1-\delta}}\right)^{\mu_{\min}/R} \text{ for } \delta \in [0,1], and \\
\Pr\left\{\lambda_{\max}(\sum_k  X_k) \ge (1+\delta)\mu_{\max}\right\} & \le d \left( \frac{e^{-\delta}}{(1+\delta)^{1+\delta}}\right)^{\mu_{\max}/R} \text{ for } \delta \ge 0 
\end{align}
\end{theorem}

\begin{proof}
Let $v_1,\ldots,v_d$ be the orthonormal basis of $\FF$ in the definition of matrix $A$. For any $h \in \FF$, let $\alpha(h)=(\alpha_1,\ldots,\alpha_d)$ denote the coefficients of $h$ under $v_1,\ldots,v_d$ such that $\|h\|_D^2=\|\alpha(h)\|_2^2$. At the same time, for any fixed $x$, $\underset{h \in \FF}{\sup} \frac{|h^{(D')}(x)|^2}{\|h\|_D^2}=\underset{\alpha(h)}{\sup} \frac{|\sum_{i=1}^d \alpha(h)_i \cdot v^{(D')}_i(x)|^2}{\|\alpha(h)\|_2^2} = \sum_{i \in [d]} |v_i^{(D')}(x)|^2$ by the tightness of the Cauchy Schwartz inequality. Thus 
\begin{equation}\label{def:new_K_D}
K_{D'} \overset{def}{=}\underset{x \in G}{\sup} \big\{ \underset{h \in \FF: h \neq 0}{\sup} \frac{|h^{(D')}(x)|^2}{\|h\|_D^2} \big\} \quad \text{ indicates } \quad \sup_{x \in G} \sum_{i \in [d]} |v_i^{(D')}(x)|^2 \le K_{D'}. %
\end{equation}
For each point $x_j$ in $S$ with weight $w_j=\frac{D(x_j)}{m \cdot D'(x_j)}$, let $A_j$ denote the $j$th row of the matrix $A$. It is a vector in $\mathbb{C}^{d}$ defined by 
$A_j(i)=A(j,i)=\sqrt{w_j} \cdot v_i(x_j)=\frac{v_i^{(D')}(x_j)}{\sqrt{m}}.$
So $A^* A=\sum_{j=1}^m  A_j^* \cdot A_j$. 

For $A_j^* \cdot A_j$, it is always $\succeq 0$. Notice that the only non-zero eigenvalue of $A_j^* \cdot A_j$ is
$$
\lambda(A_j^* \cdot A_j)= A_j \cdot A_j^* =\frac{1}{m} \left(\sum_{i \in [d]} |v_i^{(D')}(x_j)|^2 \right) \le \frac{K_{D'}}{m}
$$ from \eqref{def:new_K_D}.

At the same time, $\sum_{j=1}^m \E[A_j^* \cdot A_j]$ equals the identity matrix of size $d \times d$ because the expectation of the entry $(i,i')$ in $A_j^* \cdot A_j$ is
\begin{align*}
\underset{x_j \sim D'}{\E}[\overline{A(j,i)} \cdot A(j,i')]&=\underset{x_j \sim D'}{\E}[\frac{\overline{v^{(D')}_{i}(x_j)} \cdot v^{(D')}_{i'}(x_j)}{m}]\\
&=\underset{x_j \sim D'}{\E}[\frac{D(x) \cdot \overline{v_{i}(x_j)} \cdot v_{i'}(x_j)}{m \cdot D'(x_j) }]=\underset{x_j \sim D}{\E}[\frac{\overline{v_{i}(x_j)} \cdot v_{i'}(x_j)}{m}]=1_{\vec{i}=\vec{i}'}/m.
\end{align*}

Now we apply Theorem~\ref{thm:matrix_chernoff} on $A^* A=\sum_{j=1}^m (A_j^* \cdot A_j)$:
\begin{align*}
\Pr\left[\lambda(A^* A) \notin [1-\eps,1+\eps]\right] &\le d \left( \frac{e^{-\eps}}{(1-\eps)^{1-\eps}}\right)^{1/\frac{K_{D'}}{m}} + d \left( \frac{e^{-\eps}}{(1+\eps)^{1+\eps}}\right)^{1/\frac{K_{D'}}{m}}\\
&\le 2d \cdot e^{-\frac{\eps^2 \cdot \frac{m}{K_{D'}}}{3}} \le \delta \qquad \qquad \text{ given } m \ge \frac{6 K_{D'} \log \frac{d}{\delta}}{\eps^2}.
\end{align*}
\end{proof}

Then we finish the proof of Lemma~\ref{lem:agnostic_learning_single_distribution}.

\begin{proofof}{Lemma~\ref{lem:agnostic_learning_single_distribution}}
Because the coefficient $\alpha_i=1/m=O(\eps/K_{D'})$ and $\sum_i \alpha_i=1$, this indicates the second property of \emph{well-balanced sampling procedures}.

Since $m=\Theta(K_{D'} \log d)$, by Lemma~\ref{lmm:general_distribution}, we know all eigenvalues of $A^* \cdot A$ are in $[1-1/4,1+1/4]$ with probability $1-10^{-3}$. By Lemma~\ref{lem:operator_estimation}, this indicates the first property of \emph{well-balanced sampling procedures}.
\end{proofof}







\section{Results for Active Learning}\label{sec:act_learning}
In this section, we investigate the case where we do not know the distribution $D$ of $x$ and only receive random samples from $D$. We finish the proof of Theorem~\ref{thm:unknown_dist} that bounds the number of unlabeled samples by the condition number of $D$ and the number of labeled samples by $dim(\FF)$ to find the truth through $D$.

\restate{thm:unknown_dist}

For generality, we bound the number of labels using any \emph{well-balanced sampling procedure}, such that Theorem~\ref{thm:unknown_dist} follows from this lemma with the linear sample procedure in Lemma~\ref{lemma:BSS}.
\begin{lemma}\label{lem:act_learn_agnostic_procedure}
Consider any dimension $d$ linear
  space $\FF$ of functions from a domain $G$ to $\mathbb{C}$. Let
  $(D,Y)$ be a joint distribution over $G \times \mathbb{C}$ and
  $f=\underset{h \in \FF}{\arg\min} \underset{(x,y)\sim
    (D,Y)}{\E}[|y-h(x)|^2]$.

  Let
  $K=\underset{h \in \FF: h \neq 0}{\sup} \frac{\sup_{x \in
      G}|h(x)|^2}{\|h\|_D^2}$ and $P$ be a \emph{well-balanced sampling procedure} terminating in $m_p(\eps)$ rounds with probability $1-10^{-3}$ for any linear family $\FF$, measurement $D$, and $\eps$. For any $\eps>0$ and $\delta$, Algorithm~\ref{alg:ag_unknown} takes $O(K \log \frac{d}{\delta} + \frac{K}{\eps})$ unlabeled samples from $D$ and requests at most $m_p(\eps/8)$ labels to output $\wt{f}$ satisfying 
\[
\underset{x \sim D}{\E}[|\tilde{f}(x)-f(x)|^2] \le \eps \cdot \underset{(x,y)\sim (D,Y)}{\E}[|y-f(x)|^2] \text{ conditioned on } \mathcal{E}.
\]
\end{lemma}

Algorithm~\ref{alg:ag_unknown} first takes $m_0=O(K \log \frac{d}{\delta}+ K/\eps)$ unlabeled samples and defines a distribution $D_0$ to be the uniform distribution on these $m_0$ samples. Then it uses $D_0$ to simulate $D$ in $P$, i.e., it outputs the ERM of a good execution of the well-balanced sampling procedure $P$ with the linear family $\FF$, the measurement $D_0$, and $\frac{\epsilon}{8}$.

\begin{algorithm}[H]
\caption{Regression over an unknown distribution $D$}\label{alg:ag_unknown}
\begin{algorithmic}[1]
\Procedure{\textsc{RegressionUnknownDistribution}}{$\eps,\FF,D,P$}
\State Set $C$ to be a large constant and $m_0=C \cdot (K \log \frac{d}{\delta} + K/\eps)$ . 
\State Take $m_0$ unlabeled samples $x_1,\dotsc,x_{m_0}$ from $D$.
\State Let $D_0$ be the uniform distribution over $(x_1,\dotsc,x_{m_0})$.
\State Output the ERM $\tilde{f}$ of a good execution of $P$ with parameters $\FF,D_0,\epsilon/8$.\label{step:ERM}
\EndProcedure
\end{algorithmic}
\end{algorithm}

\begin{proof}
We still use $\|f\|_{D'}$ to denote $\sqrt{\underset{x \sim D'}{\E}[|f(x)|^2]}$ and $D_1$ to denote the weighted distribution generated by Procedure $P$ given $\FF,D_0,\eps$. By Lemma~\ref{lem:agnostic_learning_single_distribution} with $D$ and the property of $P$, with probability at least $1- \delta$, random event $\mathcal{E}$ in \eqref{eq:def_event} holds and 
\begin{equation}\label{eq:D_D_0}
\|h\|^2_{D_1} = (1 \pm 1/4) \cdot \|h\|^2_{D_0} \text{ for every } h \in \FF.
\end{equation} 
We assume both $\mathcal{E}$ and \eqref{eq:D_D_0} hold in the rest of this proof.

Let $y_i$ denote a random label of $x_i$ from $Y(x_i)$ for each $i \in [m_0]$ including the unlabeled samples in the algorithm and the labeled samples in Step~\ref{step:ERM} of Algorithm~\ref{alg:ag_unknown}. Let $f'$ be the weighted ERM of $(x_1,\cdots,x_m)$ and $(y_1,\cdots,y_m)$ over $D_0$, i.e.,
\begin{equation}\label{eq:def_f'}
f'=\argmin_{h \in \FF} \underset{x_i \sim D_0,y_i \sim Y(x_i)}{\E} \left[ |y_i-h(x_i)|^2 \right].
\end{equation}

Given Property \eqref{eq:D_D_0} and Lemma~\ref{lem:agnostic_learning_single_distribution}, $$\underset{(x_1,y_1),\dotsc,(x_{m_0},y_{m_0})}{\E}[\|f'-f\|^2_D] \le \eps \cdot \underset{(x,y)\sim (D,Y)}{\E}[|y-f(x)|^2] \text{ from the proof of Theorem~\ref{thm:guarantee_AL_procedure} }.
$$


In the rest of this proof, we show that the weighted ERM $\wt{f}$ of a good execution of $P$ with measurement $D_0$ guarantees $\|\wt{f}-f'\|^2_{D_0} \lesssim \underset{(x,y) \sim (D,Y)}{\E} \left[ |y-f(x)|^2 \right]$ with high probability. Given Property \eqref{eq:D_D_0} and the guarantee of Procedure $P$, we have $$\underset{P}{\E}[\|\wt{f}-f'\|^2_{D_0}] \le \eps \cdot \underset{x \sim D_0}{\E} \left[ |y_i-f'(x_i)|^2 \right]$$ from the proof of Theorem~\ref{thm:guarantee_AL_procedure}. 
Next we bound the right hand side $\underset{x_i \sim D_0}{\E} \left[ |y_i-f'(x_i)|^2 \right]$ by $\underset{(x,y) \sim (D,Y)}{\E} \left[ |y-f(x)|^2 \right]$ over the randomness of $(x_1,y_1),\dotsc,(x_{m_0},y_{m_0})$:
\begin{align*}
& \E_{(x_1,y_1),\dotsc,(x_{m_0},y_{m_0})} \left[ \E_{x_i \sim D_0} \left[ |y_i-f'(x_i)|^2 \right] \right] \\
\le & \E_{(x_1,y_1),\dotsc,(x_{m_0},y_{m_0})} \left[2\E_{x_i \sim D_0} \left[ |y_i-f(x_i)|^2 \right] + 2 \|f-f'\|_{D_0}^2 \right]\\
\le & 2 \E_{(x,y) \sim (D,Y)} \left[ |y-f(x)|^2 \right] + 3 \E_{(x_1,y_1),\dotsc,(x_{m_0},y_{m_0})}\big[\|f-f'\|_D^2\big] \quad \text{ from }  \eqref{eq:D_D_0}
\end{align*}
Hence $\underset{(x_1,y_1),\dotsc,(x_{m_0},y_{m_0})}{\E}\big[\underset{P}{\E}[\|\wt{f}-f'\|^2_{D_0}] \big] \lesssim \eps \cdot \underset{(x,y) \sim (D,Y)}{\E} \left[ |y-f(x)|^2 \right]$.


From all discussion above, by rescaling $\eps$, we have
\[
\|\wt{f}-f\|^2_D \le 2 \|\wt{f}-f'\|^2_D + 2 \|f'-f\|^2_D \le 3 \|\wt{f}-f'\|^2_{D_0} + \frac{\eps}{4} \cdot \underset{(x,y)\sim (D,Y)}{\E}[|y-f(x)|^2]\le \eps \cdot \underset{(x,y)\sim (D,Y)}{\E}[|y-f(x)|^2]
\]
\end{proof}

\section{Lower Bounds}\label{sec:lower_bound} 
We present two lower bounds on the number of samples in this section. We first prove a lower bound on the query complexity based on the dimension $d$. Then we prove a lower bound on the the sample complexity based on the condition number of the sampling distribution.
\begin{theorem}\label{thm:information_lower_bound}
For any $d$ and any $\eps<\frac{1}{10}$, there exist a distribution $D$ and a linear family $\FF$ of functions with dimension $d$ such that for the i.i.d.~Gaussian noise $g(x)=N(0,\frac{1}{\eps})$, any algorithm which observes $y(x)=f(x)+g(x)$ for $f \in \FF$ with $\|f\|_D=1$ and outputs $\wt{f}$ satisfying $\|f-\wt{f}\|_D \le 0.1 $ with probability $\ge \frac{3}{4}$, needs at least $m \ge \frac{0.8 d}{\eps}$ queries.
\end{theorem}
Notice that this lower bound matches the upper bound in Theorem~\ref{thm:BSS} up to a constant factor. In the rest of this section, we focus on the proof of Theorem~\ref{thm:information_lower_bound}. Let  $\FF=\{f: [d] \rightarrow \mathbb{R}\}$ and $D$ be the uniform distribution over $[d]$. We first construct a packing set $\mathcal{M}$ of $\FF$.

\begin{claim}
There exists a subset $\mathcal{M}=\{f_1,\dotsc,f_n\} \subseteq \FF$ with the following properties:
\begin{enumerate}
\item $\|f_i\|_D = 1$ for each $f_i \in \mathcal{M}$.
\item $\|f_i\|_{\infty} \le 1$ for each $f_i \in \mathcal{M}$.
\item $\|f_i-f_j\|_D > 0.2$ for distinct $f_i,f_j$ in $\mathcal{M}$.
\item $n \ge 2^{0.7d}$.
\end{enumerate}
\end{claim}
\begin{proof}
We construct $\mathcal{M}$ from $U=\big\{f: [d] \rightarrow \{\pm 1\}\big\}$ in Procedure \textsc{ConstructM}. Notice that $|U|=2^d$ before the while loop. At the same time, Procedure \textsc{ConstructM} removes at most $\binom{d}{\le 0.01 d} \le 2^{0.3 d}$ functions every time because $\|g-h\|_D<0.2$ indicates $\Pr[g(x)\neq h(x)] \le (0.2)^2 /4=0.01$. Thus $n \ge 2^{d}/2^{0.3 d} \ge 2^{0.7 d}$.

\begin{algorithm}[H]
\caption{Construct $\mathcal{M}$}\label{alg:construction}
\begin{algorithmic}[1]
\Procedure{\textsc{ConstructM}}{$d$}
\State Set $n=0$ and $U=\big\{f: [d] \rightarrow \{\pm 1\}\big\}$. 
\While{$U \neq \emptyset$}
\State Choose any $h \in U$ and remove all functions $h' \in U$ with $\|h-h'\|_D<0.2$.
\State $n=n+1$ and $f_n=h$.
\EndWhile
\State Return $\mathcal{M}=\{f_1,\dotsc,f_n\}$.
\EndProcedure
\end{algorithmic}
\end{algorithm}
\end{proof}

We finish the proof of Theorem~\ref{thm:information_lower_bound} using the Shannon-Hartley theorem.
\begin{theorem}[The Shannon-Hartley Theorem \cite{Hartley,Shannon49}]\label{thm:Shannon_Hartley}
Let $S$ be a real-valued random variable with $\E[S^2]=\tau^2$ and $T \sim N(0,\sigma^2)$. The mutual information $I(S;S+T)\le \frac{1}{2} \log (1+\frac{\tau^2}{\sigma^2})$.
\end{theorem}

\begin{proofof}{Theorem~\ref{thm:information_lower_bound}}
Because of Yao's minimax principle, we assume $A$ is a deterministic algorithm given the i.i.d.~Gaussian noise. Let $I(\wt{f};f_j)$ denote the mutual information of a random function $f_j \in \mathcal{M}$ and $A$'s output $\wt{f}$ given $m$ observations $(x_1,y_1),\dotsc,(x_m,y_m)$ with $y_i=f_j(x_i)+N(0,\frac{1}{\eps})$. When the output $\wt{f}$ satisfies $\|\wt{f}-f_j\|_D \le 0.1$, $f_j$ is the closest function to $\wt{f}$ in $\mathcal{M}$ from the third property of $\mathcal{M}$. From Fano's inequality \cite{Fano}, $H(f_j|\tilde{f}) \le H(\frac{1}{4}) + \frac{\log( |\mathcal{M}| - 1)}{4}$. This indicates
$$
I(f_j;\wt{f}) = H(f_j)-H(f_j|\tilde{f}) \ge \log |\mathcal{M}| - 1 - \log( |\mathcal{M}| - 1)/4 \ge 0.7 \log |\mathcal{M}| \ge 0.4 d.
$$
At the same time, by the data processing inequality, the algorithm $A$ makes $m$ queries $\big(x_1,\dotsc,x_m\big)$ and sees $\big(y_1,\dotsc,y_m\big)$, which indicates
\begin{equation}\label{eq:information_samples}
I(\wt{f};f_j) \le I\bigg(\big(y_1,\dotsc,y_m\big);f_j\bigg)=\sum_{i=1}^m I\bigg(y_i; f_j(x_i)\big|y_1,\cdots,y_{i-1}\bigg).
\end{equation} 
For the query $x_i$, let $D_{i,j}$ denote the distribution of $f_j \in \mathcal{M}$ in the algorithm $A$ given the first $i-1$ observations $\big(x_1,y_1\big),\dotsc,\big(x_{i-1},y_{i-1}\big)$. We apply Theorem~\ref{thm:Shannon_Hartley} on $D_{i,j}$ such that it bounds 

\begin{align*}
I\left(y_i=f_j(x_i)+N(0,\frac{1}{\eps}); f_j(x_i)\big|y_1,\cdots,y_{i-1}\right) \le & \frac{1}{2} \log \left(1+\frac{\underset{f \sim D_{i,j}}{\E}[f(x_i)^2]}{1/\epsilon} \right) \\
\le & \frac{1}{2} \log \left(1+\frac{\underset{f \in \mathcal{M}}{\max}[f(x_i)^2]}{1/\eps}\right)\\
= & \frac{1}{2} \log \big(1+ \eps\big) \le \frac{\eps}{2},
\end{align*}
where we apply the second property of $\mathcal{M}$ in the second step to bound $f(x)^2$ for any $f \in \mathcal{M}$. Hence we bound $\sum_{i=1}^m I(y_i;f_j|y_1,\cdots,y_{i-1})$ by $m \cdot \frac{\eps}{2}$. This implies
$$
0.4 d \le m \cdot \frac{\eps}{2} \Rightarrow m \ge \frac{0.8 d}{\eps}. 
$$
\end{proofof}



Next we consider the sample complexity of linear regression.
\begin{theorem}\label{lem:sample_complexity}
For any $K$, $d$, and $\eps>0$, there exist a distribution $D$, a linear family of functions $\FF$ with dimension $d$ whose condition number $\underset{h \in \FF:h \neq 0}{\sup} \bigg\{ \underset{x \in G}{\sup} \frac{|h(x)|^2}{\|h\|_D^2} \bigg\}$ equals $K$, and a noise function $g$ orthogonal to $V$ such that any algorithm observing $y(x)=f(x)+g(x)$ of $f \in \FF$ needs at least $\Omega(K \log d+ \frac{K}{\eps})$ samples from $D$ to output $\wt{f}$ satisfying $\|\wt{f}-f\|_D \le 0.1 \sqrt{\eps} \cdot \|g\|_D$ with probability $\frac{3}{4}$.
\end{theorem}
\begin{proof}
We fix $K$ to be an integer and set the domain of functions in $\FF$ to be $[K]$. We choose $D$ to be the uniform distribution over $[K]$. Let $\FF$ denote the family of functions $\big\{f:[K] \rightarrow \mathbb{C}|f(d+1)=f(d+2)=\cdots=f(K)=0\big\}$. Its condition number $\underset{h \in \FF:h \neq 0}{\sup} \big\{ \underset{x \in G}{\sup} \frac{|h(x)|^2}{\|h\|_D^2} \big\}$ equals $K$. $h(x)=1_{x=1}$ provides the lower bound $\ge K$. At the same time, $\frac{|h(x)|^2}{\|h\|_D^2} = \frac{|h(x)|^2}{\sum_{i=1}^K |h(x)|^2/K} \le K$ indicates the upper bound $\le K$.

We first consider the case $K \log d \ge \frac{K}{\eps}$. Let $g=0$ such that $g$ is orthogonal to $V$. Notice that $\|\wt{f}-f\|_D \le 0.1 \sqrt{\eps} \cdot \|g\|_D$ indicates $\wt{f}(x)=f(x)$ for every $x \in [d]$. Hence the algorithm needs to sample $f(x)$ for every $x \in [d]$ when sampling from $D$: the uniform distribution over $[K]$. From the lower bound of the coupon collector problem, this takes at least $\Omega(K \log d)$ samples from $D$.

Otherwise, we prove that the algorithm needs $\Omega(K/\eps)$ samples. Without loss of generality, we assume $\underset{x \sim [d]}{\E} \big[|f(x)|^2 \big]=1$ for the truth $f$ in $y$. Let $g(x)=N(0,1/\eps)$ for each $x \in [d]$. From Theorem~\ref{thm:information_lower_bound}, to find $\wt{f}$ satisfying $\underset{x \sim [d]}{\E} \big[|\wt{f}(x)-f(x)|^2 \big] \le 0.1 \underset{x \sim [d]}{\E} \big[|f(x)|^2 \big]$, the algorithm needs at least $\Omega(d/\eps)$ queries of $x \in [d]$. Hence it needs $\Omega(K/\eps)$ random samples from $D$, the uniform distribution over $[K]$.
\end{proof}

\section[]{Application to Continuous $k$-sparse Fourier Transforms}\label{sec:k_sparse_FT}
We consider the nonlinear function space containing signals with $k$-sparse Fourier transform in the continuous setting. Let $D$ be the uniform distribution over $[-1,1]$ and $F$ be the bandlimit of the frequencies. We fix the family $\FF$ in this section to be
\[
\FF=\left\{ f(x)=\sum_{j=1}^k v_j e^{2 \pi \i \cdot f_j x}\bigg|v_j \in \mathbb{C}, |f_j| \le F\right\}.
\]
The main result in this section is an estimation of the importance sampling of $x \in [-1,1]$.

\restate{lem:bound_k_sparse_FT_x_middle}

This directly improves $\kappa=\underset{x \in [-1,1]}{\E}[\underset{f \in \FF}{\sup}
\frac{|f(x)|^2}{\|f\|_D^2}]$ for signals with $k$-sparse Fourier
transform, which is better than the condition number
$\underset{x \in [-1,1]}{\sup}[\underset{f \in \FF}{\sup}
\frac{|f(x)|^2}{\|f\|_D^2}]$ used in \cite{CKPS}.

\begin{theorem}\label{thm:kappa_sparse_FT}
For signals with $k$-sparse Fourier transform, 
 
\[
\underset{x \in [-1,1]}{\E} \left[\underset{f \in \FF}{\sup} \frac{|f(x)|^2}{\|f\|_D^2} \right] = O(k \log^2 k).
\]
Moreover, there exists a constant $c=\Theta(1)$ such that a distribution
\begin{flalign*}
& D_{\FF}(x)=
\begin{cases}
\frac{c}{(1-|x|) \log k }, & \text{ for } |x| \le 1-\frac{1}{k^3 \log^2 k}\\
c \cdot k^3 \log k, & \text{ for } |x| > 1-\frac{1}{k^3 \log^2 k}
\end{cases} 
\\
\text{guarantees for any } & f(x)=\sum_{j=1}^k v_j e^{2 \pi \i f_j x} \text{ and any } x \in [-1,1], \quad |f(x)|^2 \cdot \frac{D(x)}{D_{\FF}(x)} = O(k \log^2 k) \cdot \|f\|_D^2.
\end{flalign*}
\end{theorem}

We first state the condition number result in the previous work \cite{CKPS}.
\begin{lemma}[Lemma~5.1 of \cite{CKPS}]\label{lem:sparse_FT_infinity}
For any $f(x)=\sum_{j=1}^k v_j e^{2 \pi \i f_j x}$, 
\[
\underset{x \in [-1,1]}{\sup} \frac{|f(x)|^2}{\|f\|_D^2}=O(k^4 \log^3 k).
\]
\end{lemma}

We first show an interpolation lemma of $f(x)$ then finish the proof of Theorem~\ref{lem:bound_k_sparse_FT_x_middle}.
\begin{claim}
Given $f(x)=\sum_{j=1}^k v_j e^{2 \pi \bi f_j \cdot x}$ and $\Delta$, there exists $l \in [2k]$ such that for any $t$, 
$$
|f(t + l \cdot \Delta)|^2 \lesssim \sum_{j \in [2k] \setminus \{l\}} |f(t+j \cdot \Delta)|^2.
$$
\end{claim}
\begin{proof}
Given $k$ frequencies $f_1,\ldots,f_k$ and $\Delta$, we set $z_1=e^{2 \pi \bi f_1 \cdot \Delta},\ldots,z_k=e^{2 \pi \bi f_k \cdot \Delta}$. Let $V$ be the linear subspace 
$$
\left\{ \big(\alpha(0),\ldots,\alpha(2k-1)\big) \in \mathbb{C}^{2k} \big| \sum_{j=0}^{2k-1} \alpha(j) \cdot z_i^j=0, \forall i \in [k] \right).
$$
Because the dimension of $V$ is $k$, let $\alpha_1,\ldots,\alpha_k \in V$ be $k$ orthogonal coefficient vectors with unit length $\|\alpha_i\|_2=1$. From the definition of $\alpha_i$, we have
\begin{align*}
\sum_{j \in [2k]} \alpha_i(j) \cdot f(t+j \cdot \Delta) & =\sum_{j \in [2k]} \alpha_i(j) \sum_{j' \in [k]} v_{j'} \cdot e^{2 \pi \bi f_{j'} \cdot (t+j \Delta)}\\ 
& =\sum_{j \in [2k]} \alpha_i(j) \sum_{j' \in [k]} v_{j'} \cdot e^{2 \pi \bi f_{j'} t} \cdot z_{j'}^j = \sum_{j'} v_{j'} \cdot e^{2 \pi \bi f_{j'} t} \sum_{j \in [2k]} \alpha_i(j) \cdot z_{j'}^j=0.
\end{align*}
Let $l$ be the coordinate in $[2k]$ with the largest weight $\sum_{i=1}^k |\alpha_i(l)|^2$. 
For every $i \in [k]$, from the above discussion, 
\begin{equation}\label{eq:f_linear_relation}
-\alpha_i(l) \cdot f(t+l \cdot \Delta)=\sum_{j \in [2k] \setminus \{l\}} \alpha_i(j) \cdot f(t+j \cdot \Delta).
\end{equation}
Let $A \in \mathbb{R}^{[k] \times [2k-1]}$ denote the matrix of the coefficients excluding the coordinate $l$, i.e., 
$$
A = 
 \begin{pmatrix}
  \alpha_{1}(0) & \cdots & \alpha_1(l-1) & \alpha_1(l+1) & \cdots & \alpha_{1}(2k-1) \\
  \alpha_{2}(0) & \cdots & \alpha_2(l-1) & \alpha_2(l+1) & \cdots & \alpha_{2}(2k-1) \\
  \vdots  & \vdots  & \vdots & \vdots  & \vdots & \vdots \\
  \alpha_{k}(0) & \cdots & \alpha_k(l-1) & \alpha_k(l+1) & \cdots & \alpha_{k}(2k-1) \\
 \end{pmatrix}.
$$
For the $k \times k$ matrix $A \cdot A^*$, its entry $(i,i')$ equals 
$$
\sum_{j \in [2k]\setminus \{l\}} \alpha_i(j) \cdot \overline{\alpha_{i'}(j)}=\langle \alpha_i, \alpha_{i'} \rangle - \alpha_i(l) \cdot \overline{\alpha_{i'}(l)}=1_{i=i'} - \alpha_i(l) \cdot \overline{\alpha_{i'}(l)}.
$$
Thus the eigenvalues of $A \cdot A^*$ are bounded by $1 + \sum_{i \in [k]} |\alpha_i(l)|^2$, which also bounds the eigenvalues of $A^* \cdot A$ by $1 + \sum_{i \in [k]} |\alpha_i(l)|^2$. From \eqref{eq:f_linear_relation}, 
\begin{align*}
\sum_{i \in [k]} |\alpha_i(l) \cdot f(t+l \cdot \Delta)|^2 & \le \lambda_{\max}(A^* \cdot A) \cdot \sum_{j \in [2k] \setminus \{l\}} |f(t+j \cdot \Delta)|^2  \\
\Rightarrow \big(\sum_{i \in [k]} |\alpha_i(l)|^2\big) \cdot |f(t+l \cdot \Delta)|^2 & \le (1 + \sum_{i \in [k]} |\alpha_i(l)|^2) \cdot \sum_{j \in [2k] \setminus \{l\}} |f(t+j \cdot \Delta)|^2.
\end{align*}
Because $l=\argmax_{j \in [2k]} \big\{ \sum_{i \in [k]} |\alpha_i(j)|^2 \big\}$ and $\alpha_1,\ldots,\alpha_k$ are unit vectors, $\sum_{i \in [k]} |\alpha_i(l)|^2 \ge \sum_{i=1}^k \|\alpha_i\|_2^2 /2k \ge 1/2$. Therefore
\[
|f(t+l \cdot \Delta)|^2 \le 3 \sum_{j \in [2k] \setminus \{l\}} |f(t+j \cdot \Delta)|^2.
\]
\end{proof}
\begin{corollary}
Given $f(x)=\sum_{j=1}^k v_j e^{2 \pi \bi f_j \cdot x}$, for any $\Delta$ and $t$,
$$
|f(t)|^2\lesssim \sum_{i=1}^{2k} |f(t+i \Delta)|^2 + \sum_{i=1}^{2k} |f(t - i \Delta)|^2.
$$
\end{corollary}
Next we finish the proof of Theorem~\ref{lem:bound_k_sparse_FT_x_middle}.

\begin{proofof}{Theorem~\ref{lem:bound_k_sparse_FT_x_middle}}
We assume $t=1-\epsilon$ for an $\epsilon \le 1$ and integrate $\Delta$ from $0$ to $\epsilon/2k$:

\begin{align*}
\eps/2k \cdot |f(t)|^2 & \lesssim \int_{\Delta=0}^{\eps/2k} \sum_{i=1}^{2k} |f(t+i \Delta)|^2 + \sum_{i=1}^{2k} |f(t - i \Delta)|^2 \mathrm{d} \Delta\\
& = \sum_{i \in [1,\dotsc,2k]} \int_{\Delta=0}^{\eps/2k} |f(t+i \Delta)|^2 + |f(t-i \Delta)|^2 \mathrm{d} \Delta \\
& \lesssim \sum_{i \in [1,\dotsc,2k]} \frac{1}{i} \cdot \int_{\Delta'=0}^{\eps \cdot i/2k} |f(t+\Delta')|^2 \mathrm{d} \Delta' + \sum_{i \in [1,\dotsc,2k]} \frac{1}{i} \cdot \int_{\Delta'=0}^{\eps \cdot i/2k} |f(t - \Delta')|^2 \mathrm{d} \Delta'\\
& \lesssim \sum_{i \in [1,\dotsc,2k]} \frac{1}{i} \cdot \int_{\Delta'=-\eps}^{\eps} |f(t+\Delta')|^2 \mathrm{d} \Delta'\\
& \lesssim \log k \cdot \int_{x=-1}^{1} |f(x)|^2 \mathrm{d} x.
\end{align*}
From all discussion above, we have $|f(1-\eps)|^2 \lesssim \frac{k \log k}{\eps} \cdot \underset{x \in [-1,1]}{\E}[|f(x)|^2]$.
\end{proofof}

\begin{proofof}{Theorem~\ref{thm:kappa_sparse_FT}}
We bound  
\begin{align*}
\kappa& =\underset{x \in [-1,1]}{\E}[\underset{f \in \FF}{\sup} \frac{|f(x)|^2}{\|f\|_D^2}] \\
& = \frac{1}{2}\int_{x=-1}^1 \underset{f \in \FF}{\sup} \frac{|f(x)|^2}{\|f\|_D^2} \mathrm{d} x \\
& \lesssim \int_{x=-1+\eps}^{1-\eps} \underset{f \in \FF}{\sup} \frac{|f(x)|^2}{\|f\|_D^2} \mathrm{d} x + \eps \cdot k^4 \log^3 k & \text{ from Lemma~\ref{lem:sparse_FT_infinity}}\\
& \lesssim \int_{x=-1+\eps}^{1-\eps} \frac{k \log k}{1-|x|}  \mathrm{d} x + \eps \cdot k^4 \log^3 k  & \text{ from Theorem~\ref{lem:bound_k_sparse_FT_x_middle}}\\
& \lesssim k \log k \cdot \log \frac{1}{\eps} + \eps \cdot k^4 \log^3 k  \lesssim k \log^2 k  
\end{align*}
by choosing $\eps=\frac{1}{k^3 \log k}$. Next we define $D_{\FF}(x)=D(x) \cdot \frac{\underset{f \in \FF, f \neq 0}{\sup}\frac{|f(x)|^2}{\|f\|_D^2}}{\kappa}$. The description of $D_{\FF}(x)$ follows the upper bound of $\underset{f \in \FF, f \neq 0}{\sup}\frac{|f(x)|^2}{\|f\|_D^2}$ in Lemma~\ref{lem:sparse_FT_infinity} and Theorem~\ref{lem:bound_k_sparse_FT_x_middle}. From Claim~\ref{clm:new_conditioning}, its condition number is $\kappa=O(k \log^2 k)$.
\end{proofof}

Before we show a sample-efficient algorithm, we state the following version of the Chernoff bound that will used in this proof.
\begin{lemma}[Chernoff Bound \cite{chernoff1952,tarjan09} ]\label{lem:chernoff_bound}
Let $X_1, X_2, \ldots, X_n$ be independent random variables. Assume that $0\leq X_i \leq 1$ always, for each $i \in [n]$. Let $X= X_1+X_2+\cdots+X_n$ and $\mu = \mathbb{E}[X] = \overset{n}{  \underset{i=1}{\sum} } \mathbb{E}[X_i]$. Then for any $\eps>0$,
\[
\mathsf{Pr} [ X \geq (1+\eps) \mu ] \leq \exp(-\frac{\eps^2 }{2+\eps} \mu) \textit{ and } \mathsf{Pr} [ X \geq (1-\eps) \mu ] \leq \exp(-\frac{\eps^2 }{2} \mu).
\]
\end{lemma}
\begin{corollary}\label{cor:chernoff_bound}
Let $X_1, X_2, \ldots, X_n$ be independent random variables in $[0,R]$ with expectation $1$. For any $\eps<1/2$, $X=\frac{\sum_{i=1}^n X_i}{n}$ with expectation 1 satisfies 
\[
\Pr[|X-1| \ge \eps] \le 2 \exp(-\frac{\eps^2}{3} \cdot \frac{n}{R}).
\]
\end{corollary}
Finally, we provide a relatively sample-efficient algorithm to recover
$k$-Fourier-sparse signals.  Applying the same proof with uniform
samples would require a $K/\kappa = O(k^3)$ factor more samples.

\begin{corollary}\label{cor:net_k_sparse_FT}
For any $F>0,T>0,\eps>0$, and observation $y(x)=\sum_{j=1}^k v_j e^{2\pi \i f_j x} + g(x)$ with $|f_j| \le F$ for each $j$, there exists a non-adaptive algorithm that takes $m=O(k^4 \log^3 k + k^2 \log^2 k \cdot \log \frac{FT}{\eps})$ random samples $t_1,\dotsc,t_m$ from $D_{\FF}$ and outputs $\wt{f}(x)=\sum_{j=1}^k \wt{v}_j e^{2\pi \i \wt{f}_j x}$ satisfying
\[
\E_{x \sim [-T,T]} \left[|\wt{f}(x)-f(x)|^2 \right] \lesssim \E_{x \sim [-T,T]} \left[|g(x)|^2\right] + \eps \E_{x \sim [-T,T]} \left[|f(x)|^2\right] \text{ with probability 0.9}.
\] 
\end{corollary}
\begin{algorithm}[t]
\caption{Recover $k$-sparse FT}\label{alg:k_sparse_FT}
\begin{algorithmic}[1]
\Procedure{\textsc{SparseFT}}{$y,F,T,\eps$}
\State $m \leftarrow O(k^4 \log^3 k + k^2 \log^2 k \log \frac{FT}{\eps})$
\State Sample $t_1,\dotsc,t_m$ from $D_{\FF}$ independently
\State Set the corresponding weights $(w_1,\dotsc,w_m)$ and $S=(t_1,\dotsc,t_m)$
\State Query $y(t_1),\dotsc,y(t_m)$ from the observation $y$
\State $N_{f} \leftarrow \frac{\eps}{T \cdot k^{C k^2}} \cdot \mathbb{Z}  \cap [-F,F]$ for a constant $C$
\For{ all possible $k$ frequencies $f'_1,\dotsc,f'_k$ in $N_f$}
\State Find $h(x)$ in ${\text{span}\{e^{2 \pi \i \cdot f'_1 x},\dotsc,e^{2 \pi \i \cdot f'_k x}\}}$ minimizing $\|h-y\|_{S,w}$
\State Update $\wt{f}=h$ if $\|h-y\|_{S,w} \le \|\wt{f}-y\|_{S,w}$
\EndFor 
\State Return $\wt{f}$.
\EndProcedure
\end{algorithmic}
\end{algorithm}
\begin{proof}
We first state the main tool from the previous work. From Lemma 2.1 in \cite{CKPS}, let $N_{f}=\frac{\eps}{T \cdot k^{C k^2}} \cdot \mathbb{Z}  \cap [-F,F]$ denote a net of frequencies for a  constant $C$. For any signal $f(x)=\sum_{j=1}^k v_j e^{2 \pi \i f_j (x)}$, there exists a $k$-sparse signal 
$$
f'(x)=\sum_{j=1}^k v'_j e^{2 \pi \i f'_j (x)} \text{ satisfying } \|f-f'\|_D \le \eps \|f\|_D
$$
whose frequencies $f'_1,\dotsc,f'_k$ are in $N_f$. We rewrite $y=f+g=f'+g'$ where $g'=g+f-f'$ with $\|g'\|_D \le \|g\|_D + \eps \|f\|_D$. Our goal is to recover $f'$.

We construct a $\delta$-net with $\delta=0.05$ for 
\[
\left\{h(x)=\sum_{j=1}^{2k} v_j e^{2 \pi \i \cdot \wh{h}_j x} \bigg| \|h\|_D=1, \wh{h}_j \in N_f \right\}.
\]  
We first pick $2k$ frequencies $\wh{h}_1,\dotsc,\wh{h}_{2k}$ in $N_f$ then construct a $\delta$-net on the linear subspace $\text{span}\{e^{2 \pi \i \wh{h}_1 x},\dotsc,e^{2 \pi \i \wh{h}_{2k} x}\}$. Hence the size of our $\delta$ net is 
\[
{\frac{4FT \cdot k^{C k^2}}{\eps} \choose 2k} \cdot (12/\delta)^{2k} \le (\frac{4FT \cdot k^{C k^2}}{\eps \cdot \delta})^{3k}.
\]
Now we consider the number of random samples from $D_{\FF}$ to estimate signals in the $\delta$-net. Based on the condition number of $D_{\FF}$ in  Theorem~\ref{thm:kappa_sparse_FT} and the Chernoff bound of Corollary~\ref{cor:chernoff_bound}, a union bound over the $\delta$-net indicates
$$
m=O(\frac{k \log^2 k}{\delta^2} \cdot \log |\text{net}|)=O\bigg(\frac{k \log^2 k}{\delta^2} \cdot (k^3 \log k + k \log \frac{FT}{\eps \delta}) \bigg)
$$ 
random samples from $D_{\FF}$ would guarantee that for any signal $h$ in the net, $\|h\|^2_{S,w}=(1 \pm \delta) \|h\|_D^2$. From the property of the net,
\[
\text{ for any } h(x)=\sum_{j=1}^{2k} v_j e^{2 \pi \i h_j (x)} \text{ with } \wh{h}_j \in N_f, \quad \|h\|^2_{S,w}=(1 \pm 2\delta) \|h\|_D^2,
\] 
which is sufficient to recover $f'$.

We present the algorithm in Algorithm~\ref{alg:k_sparse_FT} and bound $\|f-\wt{f}\|_D$ as follows.  The expectation of $\|f-\wt{f}\|_D$ over the random samples $S=(t_1,\dotsc,t_m)$ is
\begin{align*}
\|f-f'\|_D + \|f'-\wt{f}\|_D & \le \|f-f'\|_D + 1.1 \|f'-\wt{f}\|_{S,w} \\
& \le \|f-f'\|_D + 1.1 (\|f'-y\|_{S,w} + \|y-\wt{f}\|_{S,w}) \\
& \le \|f-f'\|_D + 1.1 (\|g'\|_{S,w} + \|y-f'\|_{S,w})\\
& \le \eps \|f\|_D + 2.2 (\|g\|_D + \eps \|f\|_D).
\end{align*}
From the Markov inequality, with probability $0.9$, $\|f-\wt{f}\|_D \lesssim \eps \|f\|_D + \|g\|_D$.
\end{proof}

\section*{Acknowledgements}
The authors would like to thank Adam Klivans and David Zuckerman for many helpful comments about this work. After using Claude to verify all calculations in this work, several theorem statements have been updated for correctness. The authors assume responsibility for all content.


\bibliographystyle{alpha}
\bibliography{polynomial_interpolation}

@inproceedings{song2019relative,
  title={Relative error tensor low rank approximation},
  author={Song, Zhao and Woodruff, David P and Zhong, Peilin},
  booktitle={Proceedings of the Thirtieth Annual ACM-SIAM Symposium on Discrete Algorithms},
  pages={2772--2789},
  year={2019},
  organization={SIAM}
}

@article{mahoney2011randomized,
  title={Randomized algorithms for matrices and data},
  author={Mahoney, Michael W},
  journal={Foundations and Trends{\textregistered} in Machine Learning},
  volume={3},
  number={2},
  pages={123--224},
  year={2011},
  publisher={Now Publishers, Inc.}
}

@article{boutsidis2013near,
  title={Near-optimal coresets for least-squares regression},
  author={Boutsidis, Christos and Drineas, Petros and Magdon-Ismail, Malik},
  journal={IEEE transactions on information theory},
  volume={59},
  number={10},
  pages={6880--6892},
  year={2013},
  publisher={IEEE}
}

@inproceedings{AKMMVZ17,
  author    = {Haim Avron and
               Michael Kapralov and
               Cameron Musco and
               Christopher Musco and
               Ameya Velingker and
               Amir Zandieh},
  title     = {Random Fourier Features for Kernel Ridge Regression: Approximation
               Bounds and Statistical Guarantees},
  booktitle = {Proceedings of the 34th International Conference on Machine Learning,
               {ICML} 2017},
  pages     = {253--262},
  year      = {2017},
}

@article{rauhut2012sparse,
  title={Sparse Legendre expansions via $\ell_1$-minimization},
  author={Rauhut, Holger and Ward, Rachel},
  journal={Journal of approximation theory},
  volume={164},
  number={5},
  pages={517--533},
  year={2012},
  publisher={Elsevier}
}

@incollection{ward2015importance,
  title={Importance sampling in signal processing applications},
  author={Ward, Rachel},
  booktitle={Excursions in Harmonic Analysis, Volume 4},
  pages={205--228},
  year={2015},
  publisher={Springer}
}

@inproceedings{derezinski2017unbiased,
  title={Unbiased estimates for linear regression via volume sampling},
  author={Derezinski, Michal and Warmuth, Manfred K},
  booktitle={Advances in Neural Information Processing Systems},
  pages={3087--3096},
  year={2017}
}

@article{derezinski2018tail,
  title={Tail bounds for volume sampled linear regression},
  author={Derezinski, Michal and Warmuth, Manfred K and Hsu, Daniel},
  journal={arXiv preprint arXiv:1802.06749},
  year={2018}
}

@inproceedings{sabato2014active,
  title={Active regression by stratification},
  author={Sabato, Sivan and Munos, Remi},
  booktitle={Advances in Neural Information Processing Systems},
  pages={469--477},
  year={2014}
}

@article{allen2017near,
  title={Near-Optimal Discrete Optimization for Experimental Design: A Regret Minimization Approach},
  author={Allen-Zhu, Zeyuan and Li, Yuanzhi and Singh, Aarti and Wang, Yining},
  journal={arXiv preprint arXiv:1711.05174},
  year={2017}
}

@inproceedings{chaudhuri2015convergence,
  title={Convergence rates of active learning for maximum likelihood estimation},
  author={Chaudhuri, Kamalika and Kakade, Sham M and Netrapalli, Praneeth and Sanghavi, Sujay},
  booktitle={Advances in Neural Information Processing Systems},
  pages={1090--1098},
  year={2015}
}

@article{hsu2016loss,
  title={Loss minimization and parameter estimation with heavy tails},
  author={Hsu, Daniel and Sabato, Sivan},
  journal={The Journal of Machine Learning Research},
  volume={17},
  number={1},
  pages={543--582},
  year={2016},
  publisher={JMLR. org}
}

@article{drineas2008relative,
  title={Relative-error CUR matrix decompositions},
  author={Drineas, Petros and Mahoney, Michael W and Muthukrishnan, S},
  journal={SIAM Journal on Matrix Analysis and Applications},
  volume={30},
  number={2},
  pages={844--881},
  year={2008},
  publisher={SIAM}
}

@inproceedings{LeeSun,
 author = {Lee, Yin Tat and Sun, He},
 title = {Constructing Linear-Sized Spectral Sparsification in Almost-Linear Time},
 booktitle = {Proceedings of the 2015 IEEE 56th Annual Symposium on Foundations of Computer Science (FOCS)},
 series = {FOCS '15},
 year = {2015},
 pages = {250--269},
 numpages = {20},
 publisher = {IEEE Computer Society},
}

@article{M10,
  title={Row sampling for matrix algorithms via a non-commutative Bernstein bound},
  author={Magdon-Ismail, Malik},
  journal={arXiv preprint arXiv:1008.0587},
  year={2010}
}

@article{W14,
  title={Sketching as a tool for numerical linear algebra},
  author={Woodruff, David P},
  journal={arXiv preprint arXiv:1411.4357},
  year={2014}
}

@article{batson2012twice,
  title={Twice-ramanujan sparsifiers},
  author={Batson, Joshua and Spielman, Daniel A and Srivastava, Nikhil},
  journal={SIAM Journal on Computing},
  volume={41},
  number={6},
  pages={1704--1721},
  year={2012},
  publisher={SIAM}
}

@article{CDL13,
  title={On the stability and accuracy of least squares approximations},
  author={Cohen, Albert and Davenport, Mark A and Leviatan, Dany},
  journal={Foundations of computational mathematics},
  volume={13},
  number={5},
  pages={819--834},
  year={2013},
  publisher={Springer}
}

@article{Shannon49,
 author = {Claude Shannon},
 title = {Communication in the presence of noise},
 journal = { Proc. Institute of Radio Engineers},
  volume={37},
  number={1},
  pages={10--21},
  year={1949},
  publisher={IEEE}
}

@article{Hartley,
  title={Transmission of Information},
  author={Ralph Hartley},
  journal={Bell System Technical Journal},
  year={1928}
}

@article{tarjan09,
  title={Lecture 10: More Chernoff Bounds, Sampling, and the Chernoff + Union Bound},
  author={Robert E. Tarjan},
  journal={Princeton Class Notes, Probability and Computing},
  pages={1--9},
  year={2009}
}

@article{chernoff1952,
  title={A measure of asymptotic efficiency for tests of a hypothesis based on the sum of observations},
  author={Chernoff, Herman},
  journal={The Annals of Mathematical Statistics},
  volume={23},
  pages={493--507},
  year={1952}
}

@inproceedings{Moitra15,
  title={The threshold for super-resolution via extremal functions},
  author={Moitra, Ankur},
  booktitle = {STOC},
  year = {2015}
}

@ARTICLE{BM86, 
author={Y. Bresler and A. Macovski}, 
journal={IEEE Transactions on Acoustics, Speech, and Signal Processing}, 
title={Exact maximum likelihood parameter estimation of superimposed exponential signals in noise}, 
year={1986}, 
volume={34}, 
number={5}, 
pages={1081-1089}, 
doi={10.1109/TASSP.1986.1164949}, 
ISSN={0096-3518}, 
month={Oct},}

@article{HilbertMV,
  title={Hilbert's Inequality},
  author={H.L. Montgomery and R.C. Vaughan},
  journal={Journal of the London Mathematical Society},
  number={1},
  volume={s2-8},
  pages={73--82},
  year={1974}
}

@book{Fano,
  title={Transmission of information; a statistical theory of communications},
  author={Robert Fano},
  year={1961},
  publisher={Cambridge, Massachusetts, M.I.T. Press}
}

@inproceedings{CKPS,
 author = {Xue Chen and Daniel M. Kane and Eric Price and Zhao Song},
 title = {Fourier-sparse interpolation without a frequency gap},
 booktitle = {FOCS 2016},
 year = {2016},
}

@article{Tropp,
  author    = {Joel A. Tropp},
  title     = {User-Friendly Tail Bounds for Sums of Random Matrices},
  journal   = {Foundations of Computational Mathematics},
  volume    = {12},
  pages    = {389-434},
  year      = {2012},
}








\end{document}